\pdfoutput=1
\documentclass{article}
\usepackage{iclr2027_conference,times}
\iclrfinalcopy

\usepackage[utf8]{inputenc}
\usepackage[T1]{fontenc}
\usepackage{microtype}
\usepackage{amsmath,amssymb,amsthm}
\usepackage{mathtools}
\usepackage{wrapfig}
\usepackage{booktabs}
\usepackage{siunitx}
\usepackage{enumitem}
\usepackage{xcolor}
\usepackage{algorithm}
\usepackage{algpseudocode}
\usepackage{tikz}
\usetikzlibrary{arrows.meta, positioning, fit, calc}
\usepackage{xurl}
\usepackage{hyperref}
\usepackage{titletoc}

\usepackage{amsmath,amsfonts,bm}

\def\eqref#1{equation~\ref{#1}}

\def\1{\bm{1}}

\DeclareMathAlphabet{\mathsfit}{\encodingdefault}{\sfdefault}{m}{sl}
\SetMathAlphabet{\mathsfit}{bold}{\encodingdefault}{\sfdefault}{bx}{n}

\newcommand{\Gr}{G}                                  
\newcommand{\Vs}{\mathbf{V}}                         
\newcommand{\Xs}{\mathbf{X}}                         
\newcommand{\Ys}{\mathbf{Y}}                         
\newcommand{\Rs}{\mathbf{R}}                         
\newcommand{\Mods}[1]{\mathcal{M}(#1)}               
\newcommand{\Scm}{\mathcal{M}}                       
\newcommand{\An}{\mathrm{An}}                        

\newcommand{\Do}[1]{\mathrm{do}(#1)}
\newcommand{\Px}[2]{P_{#1}(#2)}                      
\newcommand{\Qy}{Q}                                  
\newcommand{\GX}{\Gr_{\overline{\Xs}}}               
\newcommand{\Gsub}[1]{\Gr_{\overline{#1}}}

\newcommand{\ID}{\textsc{id}}
\newcommand{\FAIL}{\textsc{fail}}
\newcommand{\est}{\phi}                              

\newcommand{\Bi}[1]{B(#1)}                           
\newcommand{\Aug}{A}                                 
\newcommand{\gminus}{\ominus}                        
\newcommand{\Adm}{\mathcal{A}}                       
\newcommand{\Frontier}{\mathcal{F}}                  

\providecommand{\benchname}{\textsc{CertID}}
\newcommand{\nInstances}{1{,}200}          
\newcommand{\nVertexRange}{10 to 50}

\newcommand{\accFrontier}{97--100\%}

\newcommand{\accDecoupleAcc}{90.9\%}       
\newcommand{\accDecoupleFC}{12.2\%}        
\newcommand{\accSpread}{ten points}
\newcommand{\fcSpread}{seventeen-fold}
\newcommand{\gptScaleAcc}{99.5\%}
\newcommand{\gptScaleFC}{1.0\%}

\newcommand{\decoupleAccSmall}{96.0\%}       

\newcommand{\flipFlash}{4.5\%}
\newcommand{\flipPro}{2.5\%}
\newcommand{\flipFlashCI}{$[2.0, 7.5]$}
\newcommand{\flipProCI}{$[0.5, 5.0]$}
\newcommand{\flipPairFlash}{3.0\%}
\newcommand{\flipPairPro}{1.7\%}
\newcommand{\repeatFlipRate}{2.5--4.5\%}
\newcommand{\nRepeatInst}{200}

\newcommand{\smallModelRange}{0.5B--3B}
\newcommand{\smallWorstAUC}{0.359}
\newcommand{\smallScrambledAUC}{0.493}

\newcommand{\repairValid}{92.6--97.0\%}
\newcommand{\repairMin}{96.8--99.3\%}
\newcommand{\trivialMin}{8.3\%}

\newcommand{\justifyBare}{0\%}
\newcommand{\justifyCite}{100\%}
\newcommand{\flashRefusalBase}{97.0\%}
\newcommand{\flashRefusalAsked}{95.3\%}
\newcommand{\decoupleFCSmall}{2.0\%}

\newcommand{\estimandAcc}{98.8--98.9\%}    
\newcommand{\estimandSkip}{41--42\%}       

\newcommand{\leakAUC}{0.825}          
\newcommand{\shippedAUC}{0.652}

\newcommand{\acceptRate}{14.1\%}           

\newcommand{\nCanonical}{24}               
\newcommand{\naiveGap}{$3.8\times10^{-2}$}
\newcommand{\oracleCallTotal}{14.9M}         
\newcommand{\oracleDisagree}{0}

\newcommand{\modelFlash}{\textsc{flash}}
\newcommand{\modelPro}{\textsc{pro}}
\newcommand{\modelGPT}{\textsc{gpt-5.5}}

\newcommand{\nRefFlash}{291}
\newcommand{\nRefPro}{283}
\newcommand{\nRefGPT}{300}
\newcommand{\refMeanPro}{472}
\newcommand{\refBarePro}{74.2\%}
\newcommand{\refPninetyPro}{1{,}933}

\newcommand{\justifyLenFlash}{450.1}
\newcommand{\justifyLenPro}{501.9}
\newcommand{\justifyLenGPT}{437.3}
\newcommand{\proMedianAsked}{502}
\newcommand{\proPninetyAsked}{603}

\newcommand{\csGfromulaErr}{$2.2\times10^{-16}$}
\newcommand{\csNaiveErr}{$7.0\times10^{-3}$}
\newcommand{\csAuthorDeletes}{seven}
\newcommand{\csFrontierJoint}{five}
\newcommand{\csFrontierSingle}{eight}
\newcommand{\csLatentShare}{37\%}

\newcommand{\nSCMDraws}{2}
\newcommand{\tolValue}{10^{-9}}
\newcommand{\tikkaDev}{$8.6\times10^{-3}$}
\newcommand{\tikkaDevRange}{$4.1\times10^{-3}$ to $1.65\times10^{-2}$}

\theoremstyle{definition}
\newtheorem{definition}{Definition}
\theoremstyle{plain}
\newtheorem{theorem}{Theorem}
\newtheorem{proposition}{Proposition}
\newtheorem{lemma}{Lemma}
\newtheorem{corollary}{Corollary}
\newtheorem*{priorresult}{Hedge criterion}
\theoremstyle{remark}
\newtheorem{remark}{Remark}

\newif\ifshowchanges
\showchangesfalse

\definecolor{linkblue}{HTML}{2C5F8A}
\hypersetup{colorlinks=true, linkcolor=linkblue, citecolor=linkblue, urlcolor=linkblue,
  pdftitle={Reasoning Models Are Accurate but Unsound on Identification},
  pdfauthor={Arman Behnam, Binghui Wang}}

\title{\centering Reasoning Models Are Accurate but Unsound on Identification}

\author{
  Arman Behnam \\
  Department of Computer Science \\
  Illinois Institute of Technology \\
  Chicago, IL, USA \\
  \texttt{abehnam@hawk.illinoistech.edu} \\
  \And
  Binghui Wang \\
  Department of Computer Science \\
  Illinois Institute of Technology \\
  Chicago, IL, USA \\
  \texttt{bwang70@illinoistech.edu} \\
}

\begin{document}
\maketitle
\lhead{Preprint}

\begin{abstract}
A reasoning model asked whether a causal effect is recoverable from observational data can fail in two ways: it refuses an identifiable query or answers a non-identifiable one. The latter is more consequential, as no observational data can validate the claimed formula. Measuring this failure requires queries that are provably non-identifiable, which prior evaluations lack, and grading that accepts correct formulas in any equivalent form, which string matching cannot provide. We build \benchname{}, a formal identification pipeline that addresses both limitations. \benchname{} uses the sound and complete causal identification algorithm \textsc{id} to certify whether an effect is identifiable from a given graph and query, and verifies returned formulas against structural causal models whose interventional distributions are known exactly. \benchname{} further develops theoretical results to mitigate structural leakage, repair non-identifiable queries, and establish grading guarantees.

We evaluate three frontier reasoning models (Gemini Flash, Gemini Pro, and GPT-5.5) on \nInstances{} certified instances spanning 4 to 50 vertices. Accuracy proves a poor proxy for soundness: on identical instances, the false-claim rate on non-identifiable queries varies by \fcSpread{} across models. We also find that models decide identifiability with \accFrontier{} accuracy on graphs generated after the strongest model's training snapshot. Instances, the certification procedure, the verifier, and per-instance records are available at \url{https://anonymous.4open.science/r/certid-D718}.
\end{abstract}
\section{Introduction}
\label{sec:intro}

Evaluating causal reasoning requires fixing what counts as a correct and sound answer, and current benchmarks do so in one of two ways. \textbf{Value-computing} benchmarks~\citep{jin2023cladder, xu2026noisycausal, liu2024llms} generate questions from a fully specified structural causal model (SCM) and take the computed quantity as ground truth. Every question such a benchmark poses is answerable by construction. \textbf{Annotation-based} benchmarks~\citep{wang2024causalbench, kirichenko2026abstentionbench, lee2026econcausal} label correctness, including whether a question can be answered at all, by human judgment. Both families then score a response by matching it against a reference string. Both approaches are defensible, yet their notions of ground truth leave a specific failure mode unmeasured.

Consider a practitioner asking whether a treatment affects an outcome when the two share an unmeasured common cause (Figure~\ref{fig:overview}(a)). In this graph, the interventional quantity $P(y \mid \Do{x})$ is not identifiable: no functional of the observational distribution equals it, regardless of sample size or estimator~\citep{shpitser2006identification}. Bounds may remain recoverable, but the quantity itself cannot be recovered from observational data alone. A model that answers anyway may return an incorrect formula whose error no amount of observational data can expose. This is precisely the failure that matters in deployment, yet existing benchmark families do not measure it. Value-computing benchmarks contain only identifiable queries, while annotation-based benchmarks rely on human judgments of identifiability, on which annotators may disagree.

\begin{table}[!t]
\centering\footnotesize
\addtolength{\tabcolsep}{-2pt}
\caption{How ground truth and grading are established in causal reasoning benchmarks. In prior methods, every query is answerable when groundtruth is provided or unanswerable determined by human judgment. \benchname{} is the only benchmark without ground truth or human judgment.}

\label{tab:related}
\begin{tabular}{llll}
\toprule
Benchmark & Ground truth & Unanswerable instances & Grading \\
\midrule
\textsc{cladder}~\citep{jin2023cladder}            & oracle-computed & ---            & string \\
\textsc{noisycausal}~\citep{xu2026noisycausal}     & SCM-computed    & ---            & string \\
\textsc{qrdata}~\citep{liu2024llms}                & data-computed   & ---            & string \\
\textsc{causalbench}~\citep{wang2024causalbench}   & annotated       & ---            & string \\
\textsc{econcausal}~\citep{lee2026econcausal}      & curated         & underspecified & string \\
\textsc{abstentionbench}~\citep{kirichenko2026abstentionbench} & annotated & annotated & string \\
\textsc{crb}~\citep{sawarni2026causalreasoning}    & curated         & ---            & specification \\
\midrule
\textbf{\benchname{} (ours)} & \textbf{decided by \ID{}} & \textbf{decided by \ID{}} & \textbf{numerical} \\
\bottomrule
\end{tabular}
\end{table}

\begin{figure}[!th]
\centering
\scalebox{0.9}{
\begin{tikzpicture}[
    every node/.style={font=\small},
    var/.style={circle, draw, thick, minimum size=0.55cm, inner sep=0pt,
                font=\small},
    box/.style={rectangle, draw, rounded corners, align=center,
                minimum height=0.80cm, font=\scriptsize},
    inbox/.style={box, fill=blue!10,   minimum width=1.8cm},
    algbox/.style={box, fill=orange!14, minimum width=1.8cm},
    outbox/.style={box, fill=green!14,  minimum width=1.8cm},
    badbox/.style={box, fill=red!12, draw=red!60, minimum width=2.4cm},
    arrow/.style={-{Stealth}, thick},
    thin arrow/.style={-{Stealth}, semithick}
]
\node[font=\bfseries\small] at (1.40,3.50) {(a) The undetectable error};
\node[var] (X) at (0.70,2.40) {$X$};
\node[var] (Y) at (2.10,2.40) {$Y$};
\draw[arrow] (X) -- (Y);
\draw[{Stealth}-{Stealth}, dashed, thick, red!70] (X) to[bend left=55] (Y);
\node[font=\scriptsize, red!70] at (1.40,3.12) {unmeasured confounding};
\node[font=\scriptsize] at (1.40,1.75) {$\Qy = P(y \mid \Do{x})$ \textbf{not identifiable}};
\node[badbox] at (1.40,0.85) {model returns\\a confident estimand};
\node[font=\scriptsize, align=center, red!70] at (1.40,-0.30)
     {no observational check\\can expose the error};
\draw[arrow, red!60] (1.40,1.50) -- (1.40,1.22);

\node[font=\bfseries\small] at (7.00,3.50) {(b) Both halves computed};
\node[inbox]  (gq)    at (4.55,2.55) {graph $\Gr$, query $\Qy$};
\node[algbox] (id)    at (7.00,2.55) {\textsc{id} algorithm};
\node[outbox] (label) at (9.45,2.55) {label\\{\scriptsize is a theorem}};
\node[inbox]  (resp)  at (4.55,0.75) {model response\\{\scriptsize verdict $+$ estimand}};
\node[algbox] (scm)   at (7.00,0.75) {exact SCM\\evaluation};
\node[outbox] (grade) at (9.45,0.75) {grade\\{\scriptsize is verified}};
\draw[arrow] (gq) -- (id);
\draw[arrow] (id) -- (label);
\draw[arrow] (resp) -- (scm);
\draw[arrow] (scm) -- (grade);
\draw[thin arrow, gray] (gq) -- (resp);
\draw[thin arrow, gray, dashed] (label) -- (grade);
\node[font=\scriptsize, gray] at (7.00,3.12) {no annotator};
\node[font=\scriptsize, gray] at (7.00,0.02) {no string match};

\node[font=\bfseries\small] at (13.30,3.50) {(c) What it separates};
\node[font=\scriptsize, blue!65] at (13.30,3.12) {accuracy: spans \textbf{10 pts}};
\draw[semithick] (11.95,1.95) -- (11.95,2.85);
\draw[semithick] (11.95,1.95) -- (14.95,1.95);
\node[font=\scriptsize, left] at (11.90,2.85) {1.0};
\node[font=\scriptsize, left] at (11.90,1.95) {0};
\fill[blue!45] (12.35,1.95) rectangle (12.80,2.758);
\fill[blue!45] (13.15,1.95) rectangle (13.60,2.797);
\fill[blue!45] (13.95,1.95) rectangle (14.40,2.846);
\node[font=\scriptsize, red!70] at (13.30,1.62) {false-claim rate: spans $\mathbf{17\times}$};
\draw[semithick] (11.95,0.45) -- (11.95,1.35);
\draw[semithick] (11.95,0.45) -- (14.95,0.45);
\node[font=\scriptsize, left] at (11.90,1.35) {0.2};
\node[font=\scriptsize, left] at (11.90,0.45) {0};
\fill[red!65] (12.35,0.45) rectangle (12.80,1.215);
\fill[red!65] (13.15,0.45) rectangle (13.60,0.782);
\fill[red!65] (13.95,0.45) rectangle (14.40,0.495);
\node[font=\scriptsize] at (12.58,0.18) {A};
\node[font=\scriptsize] at (13.38,0.18) {B};
\node[font=\scriptsize] at (14.18,0.18) {C};
\node[font=\scriptsize, gray] at (13.30,-0.15) {same instances, same prompt};
\end{tikzpicture}
}
\caption{\textbf{(a)} an error current evaluations cannot catch,
\textbf{(b)} the two computations that let us catch it, and
\textbf{(c)} what the resulting measurement separates.}
\label{fig:overview}
\end{figure}
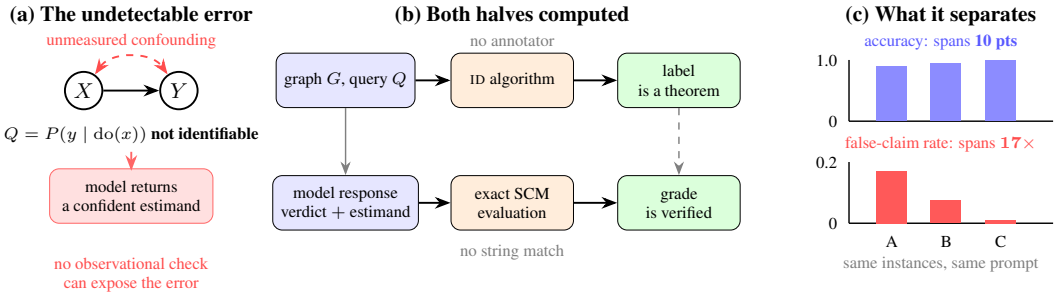

The central question we address is \textbf{whether a causal reasoning benchmark can compute its own ground truth, so that neither the label nor the grade depends on human judgment} (see Table~\ref{tab:related} for a comparison with prior work). Formal identification answers this question (Figure~\ref{fig:overview}(b)). In particular, the \ID{} identification algorithm~\citep{shpitser2006identification} can answer causal queries; we use it instead to \emph{label} them, providing formal ground truth for evaluating systems that answer such queries. This removes human judgment from ground-truth construction and enables instances to be generated at arbitrary scale after a model has been trained.

{Given an acyclic directed mixed graph $\Gr$ over observed variables $\Vs$ and a query $\Qy = P(\mathbf{y} \mid \Do{\mathbf{x}})$, \ID{} decides in time polynomial in $|\Vs|$ whether $\Qy$ is identifiable from the observational distribution. The algorithm is sound and complete~\citep{shpitser2006identification}, \emph{so its verdict is a theorem about the query, not an opinion about it}. Grading an identified estimand can likewise be reduced to computation. An estimand is a functional of the observational distribution, so we can instantiate a model consistent with the graph, evaluate the returned estimand, and compare it with the model's true interventional effect. \emph{Agreement is a calculation, not a judgment.}

Having this decision procedure, however, is not enough; two difficulties remain. First, a certified label does not ensure that a task tests what it was built to test. One of our tasks asks which confounding in a graph is real, and the variable names should carry that answer. A natural construction instead leaves structural signatures: a classifier that reads graph structure and ignores the names separates real from spurious confounding at $\mathrm{AUC} = \leakAUC{}$. We diagnose this structural leakage and remove most of it in Section~\ref{sec:leakage}. Second, exact grading requires enumerating interventional distributions with latent variables materialized, which is costly and unforgiving. Applying our verifier to a widely used causal inference library revealed two graphs for which the library returns numerically incorrect estimands even though its identifiability verdicts remain sound (Section~\ref{sec:verifier}).}

We introduce \benchname{}, a novel framework for formal causal identification. Every label follows from a complete decision procedure, and every response is graded by numerical evaluation. We further develop a leakage diagnosis and remedy, so that tasks requiring domain knowledge cannot be solved from graph structure alone. The task is to determine, given a graph and a causal query, whether the query is identifiable from observational data and, if so, to return an estimand. The input is an acyclic directed mixed graph, represented by a vertex list with directed and bi-directed edges, where each bi-directed edge represents unmeasured confounding, together with a target $P(\mathbf{y} \mid \Do{\mathbf{x}})$. The output is an identifiability verdict and, when affirmative, an estimand. Both are scored against formally certified ground truth. The same machinery applies unchanged to published causal models, e.g., on a causal diagram from \citet{robins2000marginal} (Figure~\ref{fig:robins} in Section~\ref{app:ex-published}), \benchname{} numerically confirms the authors' estimand and shows that their identifying assumption is sufficient but not minimal.

We evaluate three frontier reasoning models (Gemini Flash, Gemini Pro, and GPT-5.5) from two vendors on \nInstances{} instances. Three key findings emerge. First, identifiability is decided at \accFrontier{} on graphs that postdate the training snapshot of the strongest model tested, ruling out contamination from those graphs. Second, correct verdicts come with no justification by default, yet every model provides one when prompted, showing that observed inspectability depends on how the model is queried. Third, accuracy is a poor proxy for soundness. It has direct consequences for deployment, as Figure~\ref{fig:overview}(c) illustrates. On identical instances, the rate at which models assert answers to provably unanswerable queries varies \fcSpread{} across the three frontier models, while their accuracy varies by \accSpread{}. Accuracy alone would therefore make the three models appear similarly capable while concealing a substantial difference in how often they answer when no answer is identifiable.
\section{Background and Problem Statement}
\label{sec:background}
Random variables are uppercase ($X, Y, V$), realizations lowercase $(x,y,v)$, and sets boldface ($\Xs, \Ys, \Vs$). We write $\Gr$ for an acyclic directed mixed graph (ADMG) over observed variables $\Vs$, in which a directed edge $V_i \to V_j$ encodes direct causation, and a bi-directed edge $V_i \leftrightarrow V_j$ encodes an unobserved common cause. We write $\Bi{\Gr}$ for the bi-directed edges of $\Gr$ and $\An(\cdot)$ for ancestry. An ADMG can be obtained from a DAG with latent variables by \emph{latent projection}~\citep{verma2022equivalence}, which removes the latent variables and encodes their effects on the observed ones through directed and bi-directed edges. 

$\Mods{\Gr}$ is the set of structural causal models (SCMs) whose structural dependencies are contained in $\Gr$~\citep{pearl2009causality}, and $P(\Vs)$ is the observational distribution. Causal models can also be defined measure-theoretically, beyond discrete variables~\citep{behnam2026measure}; our verifier uses discrete SCMs. A \emph{query} is an interventional distribution $\Qy = \Px{\mathbf{x}}{\mathbf{y}} = P(\mathbf{y} \mid \Do{\mathbf{x}})$ for disjoint $\Xs, \Ys \subseteq \Vs$: the distribution of $\Ys$ when $\Xs$ is set by intervention. Complete notation appears in Appendix~\ref{app:notation}.

\begin{definition}[Identifiability]
\label{def:identifiability}
$\Qy$ is \emph{identifiable} in $\Gr$ if every pair
$\Scm_1, \Scm_2 \in \Mods{\Gr}$ inducing the same positive $P(\Vs)$ agrees on
$\Qy$; equivalently, if $\Qy$ is a functional of $P(\Vs)$ invariant across
$\Mods{\Gr}$.
\end{definition}

Identifiability is a property of $(\Gr,\Qy)$ alone, determined before any data are observed. When $\Qy$ is not identifiable, no estimator can recover it from observational data alone, regardless of sample size. This holds for learned estimators too, such as the neural causal models used to explain graph neural networks~\citep{behnam2024graph}. The same limit applies to the utility of a memory in a language model: it is identifiable only if the retriever can surface that memory~\citep{behnam2026reasoning}.

The three rules of do-calculus, i.e., inserting/deleting observations, exchanging actions with observations, and inserting/deleting actions, 
are sound~\citep{pearl1995causal} and complete~\citep{huang2006pearl} for interventional identification, but do not themselves provide a decision procedure for whether a reducing sequence exists.
The \ID{} algorithm provides such a procedure~\citep{shpitser2006identification}.

\begin{definition}[ID algorithm \citep{shpitser2006identification}]
Given $(\Gr,\Qy)$, \ID{} runs in time polynomial in $|\Vs|$ and returns either a closed-form estimand $\est$ in terms of $P(\Vs)$ or $\FAIL$ with a graphical witness of non-identifiability.
This witness is a \emph{hedge}, whose presence exactly characterizes non-identifiability.
\end{definition}
We use hedges only through \ID{}; Appendix~\ref{app:hedges} provides the definitions and criterion. Evaluating causal identification requires a correct label for each instance, a correct grade for each returned estimand, and instances that test the intended reasoning. Definitions~\ref{def:certified} and~\ref{def:verified} provide the first two, while Definition~\ref{def:evalproblem} states the evaluation problem. Sec.~\ref{sec:construction} constructs and validates the instances.

\begin{definition}[Instance and certified label]
\label{def:certified}
{An \emph{instance} is a pair $(\Gr,\Qy)$. Its \emph{certified label} is $\ell = \ID(\Gr,\Qy) \in \{\est, \FAIL\}$. Because \ID{} is sound and complete, $\ell$ is determined by $(\Gr,\Qy)$ alone and requires no human judgment. We call the triple $(\Gr,\Qy,\ell)$ a \emph{certified instance}.}
\end{definition}
String matching rejects correct estimands written in a different form and accepts wrong ones that look similar, so we grade estimands numerically.

\begin{definition}[Estimand grade]
\label{def:verified}
Let $\hat\est$ be an estimand returned by a model for an identifiable instance, and $\tau$ be a numerical tolerance. Draw $\Scm_1,\ldots,\Scm_k$ from $\Mods{\Gr}$ with all latent variables materialized, so that both $P_{\Scm_j}(\Vs)$ and $\Qy_{\Scm_j}$ can be computed exactly. The estimand is \emph{refuted} if $|\hat\est(P_{\Scm_j}(\Vs))-\Qy_{\Scm_j}|>\tau$ for some $j$, and \emph{verified over $k$ draws} otherwise.
\end{definition}

\vspace{+1mm}
\begin{definition}[Certified identification problem]
\label{def:evalproblem}
A system under evaluation maps a serialization of $(\Gr,\Qy)$ to a response
$r=(v,\hat\est)$, where $v\in\{\textsc{id},\textsc{nid},\bot\}$ is a verdict and $\hat\est$ is present only when $v=\textsc{id}$. Scoring compares the verdict $v$ with the certified label $\ell$ from Definition~\ref{def:certified}. If $\ell=\FAIL$, the response is a \emph{correct refusal} when $v=\textsc{nid}$ and a \emph{false claim} when $v=\textsc{id}$. If $\ell=\est$, the response is a \emph{correct claim} when $v=\textsc{id}$, with $\hat\est$ graded according to Definition~\ref{def:verified}, and a \emph{miss} when $v=\textsc{nid}$. Responses with $v=\bot$, for which no verdict can be recovered, are \emph{unscored}
and reported separately.
\end{definition}

The four cases---correct claim, correct refusal, miss, and false claim---are not equally costly. A miss withholds an identifiable answer that the user can obtain, e.g., by running \ID{}. A false claim supplies an answer that does not exist, and no observational data can expose it. Grading is also asymmetric: only a correct claim carries an estimand to verify, while a correct refusal is complete as stated.
\section{\benchname{}: Theoretical Foundation}
\label{sec:construction}
{Section~\ref{sec:background} defined the problem, its certified labels and estimand grade. This section forms the theoretical foundation. 
Section~\ref{sec:families} constructs instances that test the intended reasoning. Section~\ref{sec:leakage} checks that the construction does not reveal the answer through graph structure alone, the first difficulty named in Section~\ref{sec:intro}. Section~\ref{sec:repair-labels} adds the repair task on non-identifiable instances. Section~\ref{sec:grading-guarantee} proves what the estimand grade guarantees, which addresses the second difficulty, exact grading.}

\subsection{Instance Families}
\label{sec:families}
We use two graph families. In the random family, variable names are meaningless symbols, so the system must reason from graph structure alone. In the published family, graphs and variable names come from published causal diagrams, allowing the system to also draw on domain knowledge.

\begin{definition}[Random family]
\label{def:random-family}
Sample a topological order on $\Vs$, add directed edges independently with probability $p_{\mathrm{dir}}$ and bi-directed edges with probability $p_{\mathrm{bi}}$, and draw $\Xs, \Ys$ with $\Xs \subseteq \An(\Ys)$. Variable names are neutral symbols.
\end{definition}

\begin{definition}[Published family]
\label{def:semi-family}
Take a published DAG $D$ with interpretable variable names and select a latent set $L$ whose parents, if any, also belong to $L$.
Project $D$ onto $\Vs=V(D)\setminus L$, adding a bi-directed edge between each pair of observed variables that share a parent in $L$.
Because $L$ is root-closed, the projection introduces only bi-directed edges and no additional directed structure.
\end{definition}
The published family supports a second task. Alongside the real bi-directed edges induced by projection, we add spurious edges not implied by the published diagram, yielding \emph{semi-synthetic} instances. The \emph{edge-classification} task asks the system to distinguish real from spurious bi-directed edges, using variable names as the intended evidence. Appendix~\ref{app:proto-small} evaluates small open-weights language models on this task.

Identifiable and non-identifiable instances are balanced by construction, and the natural rate before balancing is reported separately, so class balance is not mistaken for a population property. Because \ID{} is polynomial, instances are generated at any size and density with no annotator in the loop.

{{\bf Oracle agreement.} Every certified label (Definition~\ref{def:certified}) is computed by two independently developed implementations of \ID{}. Across ${\sim}\oracleCallTotal{}$ backend calls, the two produced \oracleDisagree{} disagreements. Both also reproduce the published verdicts on a hand-encoded set of \nCanonical{} canonical instances (Appendix~\ref{app:verification}).}

\subsection{Structural Leakage}
\label{sec:leakage}
Edge classification should require the variable names. If graph structure alone reveals which edges are real, a system can answer correctly without domain knowledge. We hence measure such leakage.

\begin{definition}[Structural leakage]
\label{def:leakage}
The \emph{leakage} of a construction is the AUC of a classifier that distinguishes real from spurious bi-directed edges using graph structure alone. Leakage near $0.5$ indicates chance-level separability.
\end{definition}

Leakage is high when real and spurious edges are generated differently. By Definition~\ref{def:semi-family}, a latent induces a bi-directed edge between every pair of its observed children, so real edges form cliques. We therefore generate spurious edges as cliques of invented latents, structurally matched to real ones. The remaining leakage is the baseline a system must exceed to demonstrate use of variable names (Section~\ref{sec:instances}; Appendix~\ref{app:leakage}). More generally, generating the two edge types differently risks.

\subsection{The Repair Task}
\label{sec:repair-labels}
Each bi-directed edge encodes an unobserved common cause (Section~\ref{sec:background}). Assuming some causes are absent corresponds to deleting their edges, which may make a non-identifiable query identifiable. The \emph{repair task} asks for a minimal set of such deletions; every answer could be checked.

\begin{definition}[Augmentation and restoring set]
\label{def:restoring}
{An \emph{augmentation} is a set $\Aug \subseteq \Bi{\Gr}$ of bi-directed edges, and $\Gr \gminus \Aug$ is $\Gr$ with those edges deleted, asserting that the corresponding pairs share no unobserved common cause. Define
$
\Adm(\Gr,\Qy)=\{\Aug:\Qy\text{ is identifiable in }\Gr\gminus\Aug\}, \Frontier(\Gr,\Qy)=\min_{\subseteq}\Adm(\Gr,\Qy),
$
where $\Frontier(\Gr,\Qy)$ contains the inclusion-minimal \emph{restoring sets}.}
\end{definition}

\begin{proposition}[Monotonicity]
\label{prop:mono}
Deleting any set of edges preserves identifiability, so $\Adm(\Gr,\Qy)$ is upward closed; and $\Bi{\Gr} \in \Adm(\Gr,\Qy)$, since $\Gr \gminus \Bi{\Gr}$ is Markovian. Hence $\Frontier(\Gr,\Qy)$ is non-empty and characterizes $\Adm(\Gr,\Qy)$.
\end{proposition}

\begin{lemma}[Greedy minimization]
\label{lem:shrink}
Let $\Aug \in \Adm(\Gr,\Qy)$. Deleting one element at a time, whenever the result remains in $\Adm$, yields an element of $\Frontier(\Gr,\Qy)$ in $|\Aug|$ calls to \ID{}, for any deletion order.
\end{lemma}

Lemma~\ref{lem:shrink} makes every repair answer checkable: validity of a proposed $\Aug$ costs one call to \ID{}, and minimality costs $|\Aug|$ more. Augmentations are restricted to bi-directed edges for a semantic reason. Proposition~\ref{prop:mono} covers directed deletions, but deleting $X \to Y$ when that edge is the only directed path forces $P(y \mid \Do{x}) = P(y)$, which restores identifiability by assuming the answer. For a subgraph $F$ of $\Gr$, let $F \setminus \Aug$ denote deletion of the edges in $\Aug$. An augmentation $\Aug$ \emph{destroys} a hedge $(F,F')$ of $\Gr$ if $(F \setminus \Aug, F' \setminus \Aug)$ is not a hedge in $\Gr \gminus \Aug$.

\begin{theorem}[Hedge destruction]
\label{thm:destroy}
Let $\mathcal{H}(\Gr,\Qy)$ be the hedges (Appendix~\ref{app:hedges}) for $\Px{\mathbf{x}'}{\mathbf{y}'}$ in $\Gr$ over all $\Xs' \subseteq \Xs$, $\Ys' \subseteq \Ys$. Then $\Aug \in \Adm(\Gr,\Qy)$ if and only if $\Aug$ destroys every element of $\mathcal{H}(\Gr,\Qy)$. Further, $\Aug$ destroys $(F,F')$ if and only if deleting $\Aug$ disconnects the bi-directed edges of $F$ or of $F'$.
\end{theorem}

Theorem~\ref{thm:destroy} casts repair as an edge-cut problem. The closest prior result concerns interventions: \citet{akbari2022minimum} show that an intervention set restores identifiability exactly when it hits every hedge, and that the resulting design problem is NP-hard. Deleting edges instead of intervening on vertices changes the covering objects from vertices to edge cuts, so Theorem~\ref{thm:destroy} does not follow their result. When every hedge has a tree-structured bi-directed component, however, disconnection requires deleting a single edge, reducing restoration to minimum hitting set. Moreover, membership in $\Adm(\Gr,\Qy)$ is a connectivity question once the hedges are known, decidable in $O(|\Vs| + |\Bi{\Gr}|)$ per hedge and without a call to \ID{}. Appendix~\ref{app:properties-frontier} gives further results on the frontier size and the cost of deciding membership in $\Adm(\Gr,\Qy)$.

\subsection{What the Grading Guarantees}
\label{sec:grading-guarantee}
Definition~\ref{def:verified} refutes an estimand on a single numerical disagreement. A disagreement is a counterexample, so every refutation is correct. Verification rests on agreement over $k$ draws, and Theorem~\ref{thm:grading} shows that this agreement is strong evidence.

\begin{theorem}[Soundness and almost-sure refutation]
\label{thm:grading}
Fix $\Gr$ and an identifiable $\Qy$, and let $\hat\est$ be a returned estimand. Parametrize $\Mods{\Gr}$ by the conditional probability tables $\theta$ of a discrete SCM, so that $\hat\est(P_\theta(\Vs))$ and $\Qy_\theta$ are rational functions of $\theta$. Draw $\theta$ from any distribution absolutely continuous on the parameter region. With $\tau = 0$ and exact arithmetic: \emph{(i)} if the procedure refutes $\hat\est$, then $\hat\est$ does not denote $\Qy$; and \emph{(ii)} if $\hat\est$ does not agree with $\Qy$ on this family, then a single draw refutes it with probability one.
\end{theorem}
Part (i) is immediate. For part (ii), the difference of the two rational functions is not identically zero, so its zero set has measure zero, following the argument underlying randomized polynomial identity testing~\citep{schwartz1980fast,zippel1979probabilistic}. Section~\ref{sec:verifier} gives two restrictions under which our implementation compares only a slice of the target. Appendix~\ref{app:properties-grading} examines how much of Theorem~\ref{thm:grading} survives finite precision and the single intervention value used in our implementation. Proofs are deferred to Appendix~\ref{app:proofs}.
\section{\benchname{}: Evaluation Setup}
\label{sec:benchmark}
\vspace{-2mm}
Sections~\ref{sec:background} and~\ref{sec:construction} defined what we measure and why it can be trusted. This section describes the concrete run, in the three stages of Figure~\ref{fig:pipeline}. Section~\ref{sec:instances} generates and labels the instances, Section~\ref{sec:models} poses them to the systems and parses the replies, and Sections~\ref{sec:metrics} and~\ref{sec:verifier} score the responses.

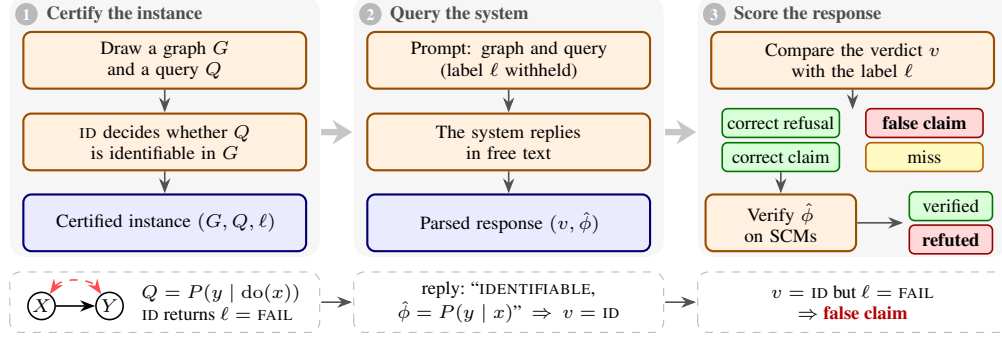
\begin{figure}[!t]
\vspace{-4mm}
\centering
\begin{tikzpicture}[
  font=\scriptsize,
  box/.style={rectangle, rounded corners=3pt, align=center, minimum height=0.75cm,
              text width=3.6cm, inner sep=2pt, thick},
  data/.style={box, fill=blue!8, draw=blue!45!black},
  proc/.style={box, fill=orange!12, draw=orange!60!black},
  chip/.style={rectangle, rounded corners=2pt, align=center, minimum height=0.38cm,
               minimum width=1.55cm, inner sep=1.5pt, semithick},
  good/.style={chip, fill=green!15, draw=green!45!black},
  warn/.style={chip, fill=yellow!30, draw=orange!70!black},
  bad/.style={chip, fill=red!15, draw=red!60!black, font=\scriptsize\bfseries},
  arr/.style={-{Stealth}, semithick, gray!65!black},
  stage/.style={-{Stealth[length=2.6mm,width=3.2mm]}, line width=2.4pt, gray!45},
  panel/.style={rounded corners=6pt, fill=gray!7},
  ex/.style={rounded corners=4pt, draw=gray!55, dashed, fill=white},
  badge/.style={circle, fill=gray!55, text=white, font=\tiny\bfseries, inner sep=1.2pt},
  ptitle/.style={anchor=west, font=\scriptsize\bfseries, text=gray!60!black},
  v/.style={circle, draw, semithick, minimum size=0.36cm, inner sep=0pt, font=\scriptsize}
]
\fill[panel] (0.00,0.95) rectangle (4.10,4.40);
\fill[panel] (4.55,0.95) rectangle (8.65,4.40);
\fill[panel] (9.10,0.95) rectangle (13.20,4.40);
\node[badge] at (0.22,4.16) {1}; \node[ptitle] at (0.36,4.16) {Certify the instance};
\node[badge] at (4.77,4.16) {2}; \node[ptitle] at (4.91,4.16) {Query the system};
\node[badge] at (9.32,4.16) {3}; \node[ptitle] at (9.46,4.16) {Score the response};
\draw[stage] (4.12,2.60) -- (4.53,2.60);
\draw[stage] (8.67,2.60) -- (9.08,2.60);
\node[proc] (gen)  at (2.05,3.55) {Draw a graph $\Gr$\\and a query $\Qy$};
\node[proc] (id)   at (2.05,2.47) {\ID{} decides whether $\Qy$\\is identifiable in $\Gr$};
\node[data] (inst) at (2.05,1.40) {Certified instance $(\Gr,\Qy,\ell)$};
\draw[arr] (gen) -- (id); \draw[arr] (id) -- (inst);
\node[proc] (prompt) at (6.60,3.55) {Prompt: graph and query\\(label $\ell$ withheld)};
\node[proc] (sys)    at (6.60,2.47) {The system replies\\in free text};
\node[data] (resp)   at (6.60,1.40) {Parsed response $(v,\hat\est)$};
\draw[arr] (prompt) -- (sys); \draw[arr] (sys) -- (resp);
\node[proc] (cmp) at (11.15,3.55) {Compare the verdict $v$\\with the label $\ell$};
\node[good] (cr) at (10.20,2.72) {correct refusal};
\node[bad]  (fc) at (12.10,2.72) {false claim};
\node[good] (cc) at (10.20,2.27) {correct claim};
\node[warn] (ms) at (12.10,2.27) {miss};
\draw[arr] (cmp.south) -- ++(0,-0.26);
\node[proc, text width=1.85cm] (ver) at (10.20,1.40) {Verify $\hat\est$\\on SCMs};
\draw[arr] (cc.south) -- (ver.north);
\node[good, minimum width=1.1cm] (ok)  at (12.45,1.62) {verified};
\node[bad,  minimum width=1.1cm] (no)  at (12.45,1.18) {refuted};
\draw[arr] (ver.east) -- (11.82,1.40);
\draw[ex] (0.00,-0.05) rectangle (4.10,0.75);
\draw[ex] (4.55,-0.05) rectangle (8.65,0.75);
\draw[ex] (9.10,-0.05) rectangle (13.20,0.75);
\draw[arr] (4.12,0.35) -- (4.53,0.35);
\draw[arr] (8.67,0.35) -- (9.08,0.35);
\node[v] (X) at (0.42,0.30) {$X$};
\node[v] (Y) at (1.32,0.30) {$Y$};
\draw[arr, black] (X) -- (Y);
\draw[{Stealth}-{Stealth}, dashed, semithick, red!70] (X) to[bend left=55] (Y);
\node[anchor=west, align=left] at (1.62,0.35) {$\Qy = P(y \mid \Do{x})$\\\ID{} returns $\ell = \FAIL$};
\node[align=center] at (6.60,0.35) {reply: ``\textsc{identifiable},\\$\hat\est = P(y \mid x)$'' \;$\Rightarrow$\; $v = \textsc{id}$};
\node[align=center] at (11.15,0.35) {$v = \textsc{id}$ but $\ell = \FAIL$\\$\Rightarrow$ {\color{red!70!black}\textbf{false claim}}};
\end{tikzpicture}
\vspace{-2mm}
\caption{{How \benchname{}'s one instance moves through the evaluation. \textbf{(1)} An instance is drawn from a family defined in Section~\ref{sec:families} and certified by \textsc{id} (Definition~\ref{def:certified}). \textbf{(2)} The system sees the graph and query but not the label, and its reply is parsed into a response (Section~\ref{sec:models}). \textbf{(3)} The verdict is compared with the label (Definition~\ref{def:evalproblem}); for a correct claim, the estimand is also verified on sampled SCMs (Section~\ref{sec:verifier}). The dashed row follows the confounded graph of Figure~\ref{fig:overview}(a): the system returns an estimand for a non-identifiable query, a false claim that no observational data can expose.}}
\label{fig:pipeline}
\vspace{-2mm}
\end{figure}

\vspace{-2mm}
\subsection{Instances}
\label{sec:instances}
\vspace{-2mm}
We instantiate the two families of Section~\ref{sec:families}. Every instance (Definition~\ref{def:certified}) pairs a graph $\Gr$ with a query $\Qy = P(y \mid \Do{x})$, where $x$ and $y$ are single vertices and $x$ is a strict ancestor of $y$. Random-family graphs (Definition~\ref{def:random-family}) fix $p_{\mathrm{dir}} = 0.35$ and sweep $p_{\mathrm{bi}}$ over $\{0.10, 0.20, 0.30, 0.40\}$ and $|\Vs|$ over $\{5, 7, 9, 11, 13\}$, with 24 instances in each of the 20 cells. Published-family graphs (Definition~\ref{def:semi-family}) project out latent sets from five published networks---\textsc{asia}, \textsc{child}, \textsc{insurance}, \textsc{alarm}, and \textsc{hepar2} (more details are in Appendix~\ref{app:networks})---and restrict the resulting graph to $\An(y)$. Each network contributes 24 instances with 4--20 vertices.

Both families are balanced by rejection: \ID{} labels each draw, and a draw is discarded once its label's quota in the cell is full. The \emph{main pool} contains 600 instances: 480 random (240 identifiable) and 120 published (60 identifiable). The \emph{scale grid} contains 600 additional random instances, 100 at each $|\Vs|\in\{10,15,20,30,40,50\}$, balanced within each size and containing at least one bi-directed edge. Edge probabilities decrease with $|\Vs|$ to maintain densities of about $1.55$ directed and $1.0$ bi-directed edges per vertex. The outcome $y$ is sampled among vertices with the largest ancestor sets, ensuring that the relevant subgraph grows with $|\Vs|$. The edge-classification task uses a separate set of 200 semi-synthetic instances (Appendix~\ref{app:leakage}). Their leakage is $\mathrm{AUC} = \shippedAUC{}$, against \leakAUC{} when spurious edges are placed at random. {\bf The verdict task runs on both pools, the repair task on their non-identifiable instances, and edge classification on the 200 semi-synthetic instances.}

\vspace{-2mm}
\subsection{Querying the Systems}
\label{sec:models}
\vspace{-2mm}
{Each instance is sent to a system as a prompt, and the system's reply is parsed into a response (Definition~\ref{def:evalproblem}). We query three frontier models from two vendors: \texttt{gemini-3.7-flash} and \texttt{gemini-3.1-pro-preview}, floating aliases that cannot be re-run against identical weights, and \texttt{gpt-5.5-2026-04-23}, a dated snapshot whose served identifier we log on every call. We call them \textsc{flash}, \textsc{pro} and \textsc{gpt-5.5}. Three open-weights models from 0.5B to 3B parameters attempt only edge classification (Appendix~\ref{app:proto-small}). Appendix~\ref{app:models} lists all identifiers.}

The prompt lists the vertices, the directed edges and the bi-directed edges, then states the query and the task. The template is shared across vendors: on the scale grid the prompts are identical in bytes, and on the main pool only the layout of the graph block differs between vendors. The repair task also has an \emph{enumeration} variant, which asks for every minimal set and is scored as recall against $\Frontier(\Gr,\Qy)$ where exhaustive enumeration is feasible. In the \emph{oracle condition}, the system may call \ID{} as a tool up to 15 times, which separates the ability to search from the ability to check.

The parser reads the verdict by three routes, in order: an explicit verdict line; an estimand block, treated as a claim; or a keyword scan. The second route yields only claims, making the routes asymmetric across labels. Because internal reasoning shares the reply token budget, we use large budgets: 16{,}384 tokens on the main pool and 32{,}768 elsewhere. Truncation can bias surviving replies toward easier instances, so we report it per cell: \textsc{flash} and \textsc{gpt-5.5} have none, while \textsc{pro} has thirteen truncated replies on the main pool and three on the scale grid. Appendix~\ref{app:truncation} gives the budgets, serializations, route counts, and results excluding truncated replies.

\vspace{-2mm}
\subsection{Estimand Verification}
\label{sec:verifier}
\vspace{-2mm}
Definition~\ref{def:verified} is realized as follows. A returned estimand is parsed into a restricted arithmetic form over conditional probabilities and evaluated against exact interventional distributions from randomly drawn discrete SCMs, with latent variables materialized and the joint distribution enumerated.

{\bf Example.} Consider the front-door graph of \citet{pearl1995causal}, $X \to M \to Y$ with $X \leftrightarrow Y$, and the query $P(y \mid \Do{x})$. Suppose a system returns $\hat\est = \sum_m P(m \mid x) \sum_{x'} P(y \mid x', m)\, P(x')$. The verifier replaces $X \leftrightarrow Y$ with an explicit binary latent $U$, draws every conditional probability table at random, and enumerates all 16 configurations of $(U, X, M, Y)$. It evaluates $\hat\est$ on the observational margin $P(X, M, Y)$, then computes the true $P(y \mid \Do{x})$ by intervening on $X$ in the same SCM. The two agree to ${\sim}10^{-16}$, so $\hat\est$ is verified. In contrast, the naive answer $P(y \mid x)$ disagrees: it absorbs the confounding carried by $U$, and therefore refuted.

We draw $k = \nSCMDraws{}$ SCMs per instance and refute when $|\hat\est(P_\Scm(\Vs)) - \Qy_\Scm| > \tau$, with $\tau = \tolValue{}$. In exact arithmetic, one draw suffices (reflected in Theorem~\ref{thm:grading}); the second guards against a floating-point coincidence on the first (see more details in Appendix~\ref{app:properties-grading}).

We calibrate the verifier in both directions: back-door and front-door estimands agree with ground truth at ${\sim}10^{-16}$, and a deliberately naive estimand is refuted at \naiveGap{}. For each library defect of Appendix~\ref{app:librarybug}, the largest deviation over the grader's draws is at least $7.5\times10^{-4}$.

Exact enumeration is exponential in the number of variables, counting one binary latent per bi-directed edge. We check an estimand only when the graph has at most 9 vertices, the joint has at most $2^{22}$ configurations, and no conditional probability table exceeds $2^{18}$ rows.

{\bf Restriction.} The comparison uses a single intervention value, so an estimand that agrees with the target at $\Do{x=1}$ but differs elsewhere is not refuted. Conditioning events with zero probability yield undefined entries, which are skipped rather than failed. Theorem~\ref{thm:grading} applies to the quantity actually compared; extending the comparison across intervention values would strengthen the verifier.

{\bf What this catches.} On five graphs, a published causal inference library returns numerically incorrect expressions despite correct identifiability verdicts. Two of the graphs come from the literature, and the formulas published for them agree with ground truth to ${\sim}10^{-16}$. This illustrates the limitation of string matching noted before Definition~\ref{def:verified}: the library's expression, the published formula, and other correct renderings differ in form (Appendices~\ref{app:librarybug} and~\ref{app:ex-estimand}), whereas numerical verification distinguishes them by what they compute.

\subsection{Metrics}
\label{sec:metrics}
We report four metrics over the scored responses defined in Definition~\ref{def:evalproblem}. For a pool, let $\mathcal{S}$ denote its scored instances, $\mathcal{I}\subseteq\mathcal{S}$ the identifiable ones, and $\mathcal{N}=\mathcal{S}\setminus\mathcal{I}$ the non-identifiable ones. {\bf Accuracy:} the fraction of $\mathcal{S}$ answered with a correct claim or correct refusal.
{\bf Identifiable rate:} the fraction $|\mathcal{I}|/|\mathcal{S}|$ of scored instances that are identifiable, equivalently the accuracy obtained by calling every query identifiable.
{\bf Answer rate:} the fraction of $\mathcal{S}$ answered with verdict \textsc{id}.
{\bf False-claim rate:} the fraction of $\mathcal{N}$ answered with a false claim.

The identifiable rate and answer rate help interpret accuracy. A system that always answers \textsc{id} achieves an accuracy equal to the identifiable rate; a system that successfully distinguishes identifiable from non-identifiable queries achieves accuracy above this baseline, while its answer rate approaches the identifiable rate. The false-claim rate isolates the consequential error in Definition~\ref{def:evalproblem}, which accuracy conflates with misses.
\section{\benchname{}: LLM Results}
\label{sec:results}

\begin{wraptable}{r}{0.55\textwidth}
\vspace{-6mm}
\centering
\footnotesize
\caption{Verdict performance on the main pool: 600 instances, balanced between identifiable and non-identifiable queries. Random-family graphs sweep $|\Vs| \in \{5,7,9,11,13\}$.}
\label{tab:verdict}
\setlength{\tabcolsep}{0pt}
\renewcommand{\arraystretch}{0.95}

\begin{tabular}{@{}lccccc@{}}
\toprule
Model & Scored & Acc. & {Id.\ rate} & Ans. rate & False-claim rate \\
\midrule
FLASH   & 600/600 & 0.983 & 0.500 & 0.513 & 0.030 \\
PRO     & 589/600 & 0.978 & 0.506 & 0.511 & 0.028 \\
GPT-5.5 & 600/600 & \textbf{0.998} & 0.500 & 0.498 & \textbf{0.000} \\
\midrule
& \multicolumn{2}{c}{Accuracy}
& & \multicolumn{2}{c}{False-claim} \\
\cmidrule(lr){2-3}\cmidrule(lr){5-6}
Model & Rand. fam. & Pub. fam. & & Rand. fam. & Pub. fam. \\
\midrule
FLASH   & 0.981 & 0.992 && 0.033 & 0.017 \\
PRO     & 0.972 & \textbf{1.000} && 0.035 & \textbf{0.000} \\
GPT-5.5 & \textbf{1.000} & 0.992 && \textbf{0.000} & \textbf{0.000} \\
\bottomrule
\end{tabular}
\vspace{-3mm}
\end{wraptable}

\vspace{-1mm}
\subsection{Frontier Models Decide Identifiability}
\label{sec:res-verdict}
From Table~\ref{tab:verdict}, all three frontier models exceed the identifiable rate, with answer rates close to it, indicating genuine discrimination. \modelPro{}'s 11 unscored instances comprise 7 main-pool truncations, 1 formatting failure, and 3 transport errors. The truncations reflect the bias noted in Section~\ref{sec:models}: each reaches the budget ceiling after roughly 15,000 tokens of internal reasoning, so the dropped instances are those the model worked longest on. \modelPro{}'s main-pool accuracy is therefore conditional on the surviving sample.

Accuracy on the random family is of primary interest because verdict must follow from graph structure alone (Section~\ref{sec:families}). For \modelGPT{}, the timeline rules out contamination: its snapshot dates to April 23, 2026, nearly 4 months before we generated the random instances on August 17, 2026, yet it answers all of them correctly.

Two observations qualify the reading. Performance differs between the families, but the published instances are naturally easier: $97.8\%$(2,716/2,777) are identifiable before balancing vs $59.4\%$(585/983) for the random family (Section~\ref{sec:families}). Thus the gap is confounded by difficulty. Gaps are small, and their sign differs between vendors. Gemini models are more accurate on the published instances, while \modelGPT{} is exact on the random instances and has a single error on published ones. Neither pattern suggests that familiarity with the published networks drives accuracy.

Difficulty is non-monotone in bi-directed density. Very sparse graphs are mostly identifiable and very dense ones mostly non-identifiable, so errors concentrate at intermediate densities, where the verdict is genuinely contested. For \modelFlash{} on the main pool, false claims peak at $0.154$ around 10 bi-directed edges and fall to 0 beyond 14 bi-directed edges.

{{\bf Estimands.} Among correctly identified instances, returned estimands are correct at \estimandAcc{} (\modelFlash{} 172/174, \modelPro{} 170/172; Sec.~\ref{sec:verifier}). Verification is limited to graphs with at most 9 vertices, leaving \estimandSkip{} of cases unchecked. \modelGPT{} estimands were not graded because its runner records verdicts without invoking the verifier.} 

{\bf Single runs are samples.} {Temperature zero did not produce deterministic outputs in Gemini models. Across 3 fresh responses for each of \nRepeatInst{} instances, \modelFlash{} is non-unanimous on \flipFlash{} and \modelPro{} on \flipPro{}, while majority voting yields no accuracy gain. \modelFlash{}'s instability occurs entirely where a flip changes a correct refusal into a false claim.}
Thus, every reported figure from these endpoints is a sample rather than a constant, and 1-2 point gap fall within rerun variability (Appendix~\ref{app:proto-repro}).

\vspace{-2mm}
\subsection{Warrants Are Withheld by Default, Not Unavailable}
\label{sec:res-warrant}
\vspace{-2mm}
A correct verdict and a checkable justification are independent, and under our default prompt the justification is usually absent. Of \modelFlash{}'s \nRefFlash{} correct refusals, \emph{every one} is exactly 25 characters: the verdict line and nothing else. \modelGPT{} is identical: all \nRefGPT{} correct refusals are exactly 25 characters.
{\modelPro{} averages \refMeanPro{} characters, but its median is also 25. The distribution is bimodal, with \modelPro{} returning a bare verdict on \refBarePro{} of its \nRefPro{} correct refusals.}

{\bf Asked directly, every model complies.} When we re-run the same non-identifiable instances with a prompt that requests the structural reason, the silence disappears. None verdict-only reply survives in any model (\justifyBare{} across $883$ correct refusals), and every reply names at least one graphical obstruction (\justifyCite{}). Mean lengths become \justifyLenFlash{}, \justifyLenPro{} and \justifyLenGPT{} characters. The shift is sharpest for \modelPro{}: its mean barely changes, but its distribution collapses, with the median rising from 25 to \proMedianAsked{} and the 90th percentile falling from \refPninetyPro{} to \proPninetyAsked{}. The cited terms track the graph, with structures irrelevant to a non-identifiable query appearing near zero (Appendix~\ref{app:truncation}).

{\bf For one model, asking costs accuracy.}
\modelFlash{}'s correct-refusal rate falls from \flashRefusalBase{} to \flashRefusalAsked{} on the same 300 instances, with false claims rising from 9 to 14. \modelPro{} improves slightly, while \modelGPT{} remains exact. The shift is roughly $1.7\times$ \modelFlash{}'s measured flip rate (Appendix~\ref{app:proto-repro}), exceeding rerun variability but not by enough for a single run to establish the effect. We report it as suggestive.

The warrant is therefore withheld, not unavailable. Both silent models reason internally at length and articulate the obstruction on request, so default silence is a property of the interface rather than a model limitation. 
{Our scope is precise: we measure what is \emph{emitted}, not whether the emitted text reflects the computation that produced the verdict. The latter is the question studied by the faithfulness literature (Appendix~\ref{sec:related-cot}), which we do not address.} 

\vspace{-2mm}
\subsection{Accuracy Is a Poor Proxy for Soundness}
\label{sec:res-decouple}
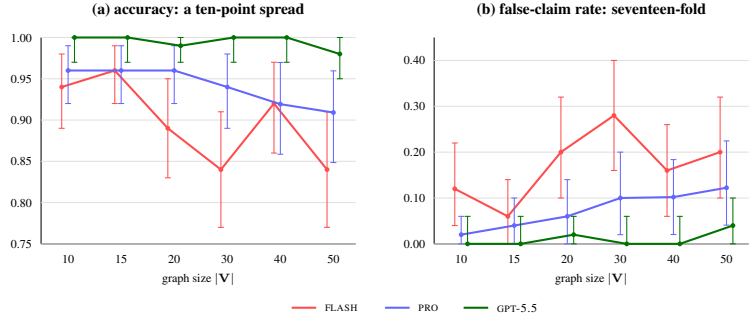
\begin{wrapfigure}{r}{0.7\textwidth}
\centering
\vspace{-4mm}
\scalebox{0.65}{
\begin{tikzpicture}[font=\small]
\node[font=\bfseries\small] at (3.2,4.75) {(a) accuracy: a ten-point spread};
\node[font=\bfseries\small] at (11.2,4.75) {(b) false-claim rate: seventeen-fold};
\draw[semithick] (0.000,0) -- (0.000,4.200);
\draw[semithick] (0.000,0) -- (6.400,0);
\draw[gray!25] (0.000,0.000) -- (6.400,0.000);
\node[font=\scriptsize,left] at (-0.060,0.000) {0.75};
\draw[gray!25] (0.000,0.840) -- (6.400,0.840);
\node[font=\scriptsize,left] at (-0.060,0.840) {0.80};
\draw[gray!25] (0.000,1.680) -- (6.400,1.680);
\node[font=\scriptsize,left] at (-0.060,1.680) {0.85};
\draw[gray!25] (0.000,2.520) -- (6.400,2.520);
\node[font=\scriptsize,left] at (-0.060,2.520) {0.90};
\draw[gray!25] (0.000,3.360) -- (6.400,3.360);
\node[font=\scriptsize,left] at (-0.060,3.360) {0.95};
\draw[gray!25] (0.000,4.200) -- (6.400,4.200);
\node[font=\scriptsize,left] at (-0.060,4.200) {1.00};
\draw[red!70,semithick] (0.420,2.352) -- (0.420,3.864);
\draw[red!70,semithick] (0.360,2.352) -- (0.480,2.352);
\draw[red!70,semithick] (0.360,3.864) -- (0.480,3.864);
\draw[red!70,semithick] (1.500,2.856) -- (1.500,4.032);
\draw[red!70,semithick] (1.440,2.856) -- (1.560,2.856);
\draw[red!70,semithick] (1.440,4.032) -- (1.560,4.032);
\draw[red!70,semithick] (2.580,1.344) -- (2.580,3.360);
\draw[red!70,semithick] (2.520,1.344) -- (2.640,1.344);
\draw[red!70,semithick] (2.520,3.360) -- (2.640,3.360);
\draw[red!70,semithick] (3.660,0.336) -- (3.660,2.688);
\draw[red!70,semithick] (3.600,0.336) -- (3.720,0.336);
\draw[red!70,semithick] (3.600,2.688) -- (3.720,2.688);
\draw[red!70,semithick] (4.740,1.848) -- (4.740,3.696);
\draw[red!70,semithick] (4.680,1.848) -- (4.800,1.848);
\draw[red!70,semithick] (4.680,3.696) -- (4.800,3.696);
\draw[red!70,semithick] (5.820,0.336) -- (5.820,2.688);
\draw[red!70,semithick] (5.760,0.336) -- (5.880,0.336);
\draw[red!70,semithick] (5.760,2.688) -- (5.880,2.688);
\draw[red!70,very thick] (0.420,3.192) -- (1.500,3.528) -- (2.580,2.352) -- (3.660,1.512) -- (4.740,2.856) -- (5.820,1.512);
\fill[red!70] (0.420,3.192) circle (1.4pt);
\fill[red!70] (1.500,3.528) circle (1.4pt);
\fill[red!70] (2.580,2.352) circle (1.4pt);
\fill[red!70] (3.660,1.512) circle (1.4pt);
\fill[red!70] (4.740,2.856) circle (1.4pt);
\fill[red!70] (5.820,1.512) circle (1.4pt);
\draw[blue!60,semithick] (0.550,2.856) -- (0.550,4.032);
\draw[blue!60,semithick] (0.490,2.856) -- (0.610,2.856);
\draw[blue!60,semithick] (0.490,4.032) -- (0.610,4.032);
\draw[blue!60,semithick] (1.630,2.856) -- (1.630,4.032);
\draw[blue!60,semithick] (1.570,2.856) -- (1.690,2.856);
\draw[blue!60,semithick] (1.570,4.032) -- (1.690,4.032);
\draw[blue!60,semithick] (2.710,2.856) -- (2.710,4.032);
\draw[blue!60,semithick] (2.650,2.856) -- (2.770,2.856);
\draw[blue!60,semithick] (2.650,4.032) -- (2.770,4.032);
\draw[blue!60,semithick] (3.790,2.352) -- (3.790,3.864);
\draw[blue!60,semithick] (3.730,2.352) -- (3.850,2.352);
\draw[blue!60,semithick] (3.730,3.864) -- (3.850,3.864);
\draw[blue!60,semithick] (4.870,1.824) -- (4.870,3.691);
\draw[blue!60,semithick] (4.810,1.824) -- (4.930,1.824);
\draw[blue!60,semithick] (4.810,3.691) -- (4.930,3.691);
\draw[blue!60,semithick] (5.950,1.655) -- (5.950,3.521);
\draw[blue!60,semithick] (5.890,1.655) -- (6.010,1.655);
\draw[blue!60,semithick] (5.890,3.521) -- (6.010,3.521);
\draw[blue!60,very thick] (0.550,3.528) -- (1.630,3.528) -- (2.710,3.528) -- (3.790,3.192) -- (4.870,2.843) -- (5.950,2.673);
\fill[blue!60] (0.550,3.528) circle (1.4pt);
\fill[blue!60] (1.630,3.528) circle (1.4pt);
\fill[blue!60] (2.710,3.528) circle (1.4pt);
\fill[blue!60] (3.790,3.192) circle (1.4pt);
\fill[blue!60] (4.870,2.843) circle (1.4pt);
\fill[blue!60] (5.950,2.673) circle (1.4pt);
\draw[green!45!black,semithick] (0.680,3.696) -- (0.680,4.200);
\draw[green!45!black,semithick] (0.620,3.696) -- (0.740,3.696);
\draw[green!45!black,semithick] (0.620,4.200) -- (0.740,4.200);
\draw[green!45!black,semithick] (1.760,3.696) -- (1.760,4.200);
\draw[green!45!black,semithick] (1.700,3.696) -- (1.820,3.696);
\draw[green!45!black,semithick] (1.700,4.200) -- (1.820,4.200);
\draw[green!45!black,semithick] (2.840,3.696) -- (2.840,4.200);
\draw[green!45!black,semithick] (2.780,3.696) -- (2.900,3.696);
\draw[green!45!black,semithick] (2.780,4.200) -- (2.900,4.200);
\draw[green!45!black,semithick] (3.920,3.696) -- (3.920,4.200);
\draw[green!45!black,semithick] (3.860,3.696) -- (3.980,3.696);
\draw[green!45!black,semithick] (3.860,4.200) -- (3.980,4.200);
\draw[green!45!black,semithick] (5.000,3.696) -- (5.000,4.200);
\draw[green!45!black,semithick] (4.940,3.696) -- (5.060,3.696);
\draw[green!45!black,semithick] (4.940,4.200) -- (5.060,4.200);
\draw[green!45!black,semithick] (6.080,3.360) -- (6.080,4.200);
\draw[green!45!black,semithick] (6.020,3.360) -- (6.140,3.360);
\draw[green!45!black,semithick] (6.020,4.200) -- (6.140,4.200);
\draw[green!45!black,very thick] (0.680,4.200) -- (1.760,4.200) -- (2.840,4.032) -- (3.920,4.200) -- (5.000,4.200) -- (6.080,3.864);
\fill[green!45!black] (0.680,4.200) circle (1.4pt);
\fill[green!45!black] (1.760,4.200) circle (1.4pt);
\fill[green!45!black] (2.840,4.032) circle (1.4pt);
\fill[green!45!black] (3.920,4.200) circle (1.4pt);
\fill[green!45!black] (5.000,4.200) circle (1.4pt);
\fill[green!45!black] (6.080,3.864) circle (1.4pt);
\node[font=\scriptsize] at (0.550,-0.30) {10};
\node[font=\scriptsize] at (1.630,-0.30) {15};
\node[font=\scriptsize] at (2.710,-0.30) {20};
\node[font=\scriptsize] at (3.790,-0.30) {30};
\node[font=\scriptsize] at (4.870,-0.30) {40};
\node[font=\scriptsize] at (5.950,-0.30) {50};
\node[font=\scriptsize] at (3.200,-0.72) {graph size $|\Vs|$};
\draw[semithick] (8.000,0) -- (8.000,4.200);
\draw[semithick] (8.000,0) -- (14.400,0);
\draw[gray!25] (8.000,0.000) -- (14.400,0.000);
\node[font=\scriptsize,left] at (7.940,0.000) {0.00};
\draw[gray!25] (8.000,0.933) -- (14.400,0.933);
\node[font=\scriptsize,left] at (7.940,0.933) {0.10};
\draw[gray!25] (8.000,1.867) -- (14.400,1.867);
\node[font=\scriptsize,left] at (7.940,1.867) {0.20};
\draw[gray!25] (8.000,2.800) -- (14.400,2.800);
\node[font=\scriptsize,left] at (7.940,2.800) {0.30};
\draw[gray!25] (8.000,3.733) -- (14.400,3.733);
\node[font=\scriptsize,left] at (7.940,3.733) {0.40};
\draw[red!70,semithick] (8.420,0.373) -- (8.420,2.053);
\draw[red!70,semithick] (8.360,0.373) -- (8.480,0.373);
\draw[red!70,semithick] (8.360,2.053) -- (8.480,2.053);
\draw[red!70,semithick] (9.500,0.000) -- (9.500,1.307);
\draw[red!70,semithick] (9.440,0.000) -- (9.560,0.000);
\draw[red!70,semithick] (9.440,1.307) -- (9.560,1.307);
\draw[red!70,semithick] (10.580,0.933) -- (10.580,2.987);
\draw[red!70,semithick] (10.520,0.933) -- (10.640,0.933);
\draw[red!70,semithick] (10.520,2.987) -- (10.640,2.987);
\draw[red!70,semithick] (11.660,1.493) -- (11.660,3.733);
\draw[red!70,semithick] (11.600,1.493) -- (11.720,1.493);
\draw[red!70,semithick] (11.600,3.733) -- (11.720,3.733);
\draw[red!70,semithick] (12.740,0.560) -- (12.740,2.427);
\draw[red!70,semithick] (12.680,0.560) -- (12.800,0.560);
\draw[red!70,semithick] (12.680,2.427) -- (12.800,2.427);
\draw[red!70,semithick] (13.820,0.933) -- (13.820,2.987);
\draw[red!70,semithick] (13.760,0.933) -- (13.880,0.933);
\draw[red!70,semithick] (13.760,2.987) -- (13.880,2.987);
\draw[red!70,very thick] (8.420,1.120) -- (9.500,0.560) -- (10.580,1.867) -- (11.660,2.613) -- (12.740,1.493) -- (13.820,1.867);
\fill[red!70] (8.420,1.120) circle (1.4pt);
\fill[red!70] (9.500,0.560) circle (1.4pt);
\fill[red!70] (10.580,1.867) circle (1.4pt);
\fill[red!70] (11.660,2.613) circle (1.4pt);
\fill[red!70] (12.740,1.493) circle (1.4pt);
\fill[red!70] (13.820,1.867) circle (1.4pt);
\draw[blue!60,semithick] (8.550,0.000) -- (8.550,0.560);
\draw[blue!60,semithick] (8.490,0.000) -- (8.610,0.000);
\draw[blue!60,semithick] (8.490,0.560) -- (8.610,0.560);
\draw[blue!60,semithick] (9.630,0.000) -- (9.630,0.933);
\draw[blue!60,semithick] (9.570,0.000) -- (9.690,0.000);
\draw[blue!60,semithick] (9.570,0.933) -- (9.690,0.933);
\draw[blue!60,semithick] (10.710,0.000) -- (10.710,1.307);
\draw[blue!60,semithick] (10.650,0.000) -- (10.770,0.000);
\draw[blue!60,semithick] (10.650,1.307) -- (10.770,1.307);
\draw[blue!60,semithick] (11.790,0.187) -- (11.790,1.867);
\draw[blue!60,semithick] (11.730,0.187) -- (11.850,0.187);
\draw[blue!60,semithick] (11.730,1.867) -- (11.850,1.867);
\draw[blue!60,semithick] (12.870,0.190) -- (12.870,1.715);
\draw[blue!60,semithick] (12.810,0.190) -- (12.930,0.190);
\draw[blue!60,semithick] (12.810,1.715) -- (12.930,1.715);
\draw[blue!60,semithick] (13.950,0.381) -- (13.950,2.095);
\draw[blue!60,semithick] (13.890,0.381) -- (14.010,0.381);
\draw[blue!60,semithick] (13.890,2.095) -- (14.010,2.095);
\draw[blue!60,very thick] (8.550,0.187) -- (9.630,0.373) -- (10.710,0.560) -- (11.790,0.933) -- (12.870,0.952) -- (13.950,1.142);
\fill[blue!60] (8.550,0.187) circle (1.4pt);
\fill[blue!60] (9.630,0.373) circle (1.4pt);
\fill[blue!60] (10.710,0.560) circle (1.4pt);
\fill[blue!60] (11.790,0.933) circle (1.4pt);
\fill[blue!60] (12.870,0.952) circle (1.4pt);
\fill[blue!60] (13.950,1.142) circle (1.4pt);
\draw[green!45!black,semithick] (8.680,0.000) -- (8.680,0.560);
\draw[green!45!black,semithick] (8.620,0.000) -- (8.740,0.000);
\draw[green!45!black,semithick] (8.620,0.560) -- (8.740,0.560);
\draw[green!45!black,semithick] (9.760,0.000) -- (9.760,0.560);
\draw[green!45!black,semithick] (9.700,0.000) -- (9.820,0.000);
\draw[green!45!black,semithick] (9.700,0.560) -- (9.820,0.560);
\draw[green!45!black,semithick] (10.840,0.000) -- (10.840,0.560);
\draw[green!45!black,semithick] (10.780,0.000) -- (10.900,0.000);
\draw[green!45!black,semithick] (10.780,0.560) -- (10.900,0.560);
\draw[green!45!black,semithick] (11.920,0.000) -- (11.920,0.560);
\draw[green!45!black,semithick] (11.860,0.000) -- (11.980,0.000);
\draw[green!45!black,semithick] (11.860,0.560) -- (11.980,0.560);
\draw[green!45!black,semithick] (13.000,0.000) -- (13.000,0.560);
\draw[green!45!black,semithick] (12.940,0.000) -- (13.060,0.000);
\draw[green!45!black,semithick] (12.940,0.560) -- (13.060,0.560);
\draw[green!45!black,semithick] (14.080,0.000) -- (14.080,0.933);
\draw[green!45!black,semithick] (14.020,0.000) -- (14.140,0.000);
\draw[green!45!black,semithick] (14.020,0.933) -- (14.140,0.933);
\draw[green!45!black,very thick] (8.680,0.000) -- (9.760,0.000) -- (10.840,0.187) -- (11.920,0.000) -- (13.000,0.000) -- (14.080,0.373);
\fill[green!45!black] (8.680,0.000) circle (1.4pt);
\fill[green!45!black] (9.760,0.000) circle (1.4pt);
\fill[green!45!black] (10.840,0.187) circle (1.4pt);
\fill[green!45!black] (11.920,0.000) circle (1.4pt);
\fill[green!45!black] (13.000,0.000) circle (1.4pt);
\fill[green!45!black] (14.080,0.373) circle (1.4pt);
\node[font=\scriptsize] at (8.550,-0.30) {10};
\node[font=\scriptsize] at (9.630,-0.30) {15};
\node[font=\scriptsize] at (10.710,-0.30) {20};
\node[font=\scriptsize] at (11.790,-0.30) {30};
\node[font=\scriptsize] at (12.870,-0.30) {40};
\node[font=\scriptsize] at (13.950,-0.30) {50};
\node[font=\scriptsize] at (11.200,-0.72) {graph size $|\Vs|$};
\draw[red!70,very thick] (5.1,-1.30) -- (5.6,-1.30);
\node[font=\scriptsize,right] at (5.65,-1.30) {\textsc{flash}};
\draw[blue!60,very thick] (7.0,-1.30) -- (7.5,-1.30);
\node[font=\scriptsize,right] at (7.55,-1.30) {\textsc{pro}};
\draw[green!45!black,very thick] (8.6,-1.30) -- (9.1,-1.30);
\node[font=\scriptsize,right] at (9.15,-1.30) {\textsc{gpt-5.5}};
\end{tikzpicture}
}
\vspace{-8mm}
\caption{Accuracy and false-claim rate against graph size, three frontier models on identical instances and identical prompts. \textbf{(a)} Accuracy separates the models by about ten points and declines gently. \textbf{(b)} The false-claim rate separates them by a factor of 17 pooled.}
\label{fig:scale}
\vspace{-4mm}
\end{wrapfigure}

Across graph sizes from \nVertexRange{} vertices, accuracy among the 3 models spans roughly 10 points, while the false-claim rate spans 17-fold (Figure~\ref{fig:scale}). \modelFlash{} and \modelPro{} differ by 4 accuracy points but 2-fold in false claims; \modelPro{} and \modelGPT{} differ by 5 points but 7-fold. An accuracy leaderboard would present these systems as close competitors while concealing the property that distinguishes them. The comparison is like-for-like in instances, prompts, and parsing: every false claim and correct refusal for \modelFlash{} and \modelGPT{} was read from an explicit verdict line, with no fallback parsing (Appendix~\ref{app:truncation}).

Two conclusions follow. First, soundness degradation with graph size is not a property of frontier models. The degradation is pronounced in one family but absent in another, at sizes where \modelFlash{} loses up to 12 points of accuracy, so one system's behavior does not predict another's. Second, even the best model is not sound. A false claim on a non-identifiable query cannot be captured downstream, and \modelGPT{} asserted an answer to three provably unanswerable queries out of 300: small, but not zero, whereas a certified procedure is zero by construction.

\subsection{Additional Results}
\label{sec:res-more}
{{\bf On a published graph:} We apply \benchname{} to the causal diagram of \citet{robins2000marginal}, verifying the authors' g-formula against exact ground truth and showing their identifying assumption is sufficient (Appendix~\ref{app:ex-published}, Figure~\ref{fig:robins}).
{\bf Truncation:} Excluding \modelPro{}'s truncated main pool replies changes no comparison (Appendix~\ref{app:proto-trunc}).
{\bf Small models:} Three open-weights models, from 0.5B to 3B parameters, perform at or below chance on edge classification (Appendix~\ref{app:proto-small}).
{\bf Scale:} Appendix~\ref{app:res-decouple} reports results by graph size with confidence intervals, and finds that larger graphs add mostly vertices irrelevant to the query.
{\bf Repair:} Without tools, \repairValid{} of proposed repairs are valid, and \repairMin{} of those are minimal. With \ID{} as a tool, every repair returned by \modelFlash{} is valid and minimal, including on graphs too large to enumerate all minimal repairs (Appendix~\ref{app:proto-repro}).}
\section{Conclusion}
\benchname{} turns causal identification into an evaluation framework with certified labels and verifiable grades. Beyond deciding identifiability, its theoretical results characterize structural leakage, repair non-identifiable queries, and provide guarantees for numerical grading. Using this framework, we find that frontier reasoning models decide identifiability accurately even on unseen random graphs, yet often provide no justification unless asked. Most importantly, we show that accuracy is a poor proxy for soundness, and that soundness must therefore be measured directly.
\newpage
\subsection*{AI use statement}

Following the categories in the ICLR 2027 AI Policy for Authors, we used generative AI tools for the following tasks with required disclosure: implementing methods, cleaning and reformatting data, and interpreting results. Among the recommended disclosure tasks, we used them for creating and editing software code, creating the artifact, drafting parts of the paper and editing it for readability, searching for and summarizing literature, and identifying relevant work. Portions of the manuscript were drafted with AI assistance and revised by the authors. We have reviewed all AI-assisted work, and we are explicit about the limits of that review. Every identifiability label was computed by two independently developed implementations of the \textsc{id} algorithm, which agreed on every instance across ${\sim}\oracleCallTotal{}$ calls, and both were checked against a set of instances whose verdicts are stated in the published literature and transcribed here from rendered primary sources.
 
\subsection*{Ethics statement}
 
This work involves no human subjects, no personally identifying data and no newly collected data. The published Bayesian networks and causal diagrams are drawn from the peer-reviewed literature and cited. \benchname{} is a measurement instrument for formal causal identification, intended for evaluation and not for deployment.
One finding concerns third-party software: Appendix~\ref{app:librarybug} reports two graphs on which a published causal inference library returns numerically incorrect estimands. We verified both against exact ground truth, and note there that the library's identifiability verdicts are correct in both cases, so the defect does not bear on its primary function. The API compute for the OpenAI experiments was from a personal account. The authors declare no competing interests arising from it.
 
\subsection*{Reproducibility statement}
 
The labeling procedure is deterministic given a graph and a query: every label is the output of the \textsc{id} algorithm, computed independently by two implementations that agreed on every instance (Sec.~\ref{sec:families}, Appendix~\ref{app:verification}). Instance generation, the leakage diagnostic and its remediation are specified in Sec.~\ref{sec:construction}, with full detail in Appendix~\ref{app:leakage}. Grading is specified in Sec.~\ref{sec:verifier}, including the number of structural causal models drawn per instance, the numerical tolerance, and the two restrictions under which the comparison is made. Proofs of all stated results are in Appendix~\ref{app:proofs}.
Model identifiers, pin quality and full revision hashes appear in Appendix~\ref{app:models}. Two of the three frontier endpoints are floating aliases that cannot be re-run against identical weights, and their verdicts are not stable under repetition (Appendix~\ref{app:proto-repro}). The open-weights results are pinned by revision hash and reproduce exactly.
Code, benchmark instances, per-instance model responses and the summary files behind every reported number are available at \url{https://anonymous.4open.science/r/certid-D718}.
\newpage

\bibliographystyle{plainnat}
\bibliography{references}

@book{pearl2009causality,
  title     = {Causality: Models, Reasoning, and Inference},
  author    = {Pearl, Judea},
  year      = {2009},
  edition   = {2nd},
  publisher = {Cambridge University Press},
  address   = {Cambridge, UK}
}

@article{pearl1995causal,
  title   = {Causal Diagrams for Empirical Research},
  author  = {Pearl, Judea},
  journal = {Biometrika},
  volume  = {82},
  number  = {4},
  pages   = {669--688},
  year    = {1995}
}

@inproceedings{tian2002general,
  title     = {A General Identification Condition for Causal Effects},
  author    = {Tian, Jin and Pearl, Judea},
  booktitle = {Proceedings of the 18th National Conference on
               Artificial Intelligence (AAAI)},
  pages     = {567--573},
  year      = {2002}
}

@inproceedings{shpitser2006identification,
  title     = {Identification of Joint Interventional Distributions in
               Recursive Semi-{M}arkovian Causal Models},
  author    = {Shpitser, Ilya and Pearl, Judea},
  booktitle = {Proceedings of the 21st National Conference on
               Artificial Intelligence (AAAI)},
  pages     = {1219--1226},
  year      = {2006}
}

@inproceedings{huang2006pearl,
  title     = {{P}earl's Calculus of Intervention Is Complete},
  author    = {Huang, Yimin and Valtorta, Marco},
  booktitle = {Proceedings of the 22nd Conference on Uncertainty in
               Artificial Intelligence (UAI)},
  pages     = {217--224},
  year      = {2006}
}

@article{shpitser2008complete,
  title   = {Complete Identification Methods for the Causal Hierarchy},
  author  = {Shpitser, Ilya and Pearl, Judea},
  journal = {Journal of Machine Learning Research},
  volume  = {9},
  pages   = {1941--1979},
  year    = {2008}
}

@article{perkovic2018complete,
  title   = {Complete Graphical Characterization and Construction of
             Adjustment Sets in {M}arkov Equivalence Classes of
             Ancestral Graphs},
  author  = {Perkovi{\'c}, Emilija and Textor, Johannes and
             Kalisch, Markus and Maathuis, Marloes H},
  journal = {Journal of Machine Learning Research},
  volume  = {18},
  number  = {220},
  pages   = {1--62},
  year    = {2018}
}

@article{tikka2017simplifying,
  title   = {Simplifying Probabilistic Expressions in Causal Inference},
  author  = {Tikka, Santtu and Karvanen, Juha},
  journal = {Journal of Machine Learning Research},
  volume  = {18},
  number  = {36},
  pages   = {1--30},
  year    = {2017}
}

@article{tikka2017identifying,
  title   = {Identifying Causal Effects with the {R} Package {causaleffect}},
  author  = {Tikka, Santtu and Karvanen, Juha},
  journal = {Journal of Statistical Software},
  volume  = {76},
  number  = {12},
  pages   = {1--30},
  year    = {2017}
}

@inproceedings{akbari2022minimum,
  title     = {Minimum Cost Intervention Design for Causal Effect
               Identification},
  author    = {Akbari, Sina and Etesami, Jalal and Kiyavash, Negar},
  booktitle = {Proceedings of the 39th International Conference on
               Machine Learning (ICML)},
  volume    = {162},
  pages     = {258--289},
  year      = {2022}
}

@article{akbari2023experimental,
  title   = {Experimental Design for Causal Effect Identification},
  author  = {Akbari, Sina and Etesami, Jalal and Kiyavash, Negar},
  journal = {arXiv preprint arXiv:2205.02232},
  year    = {2023}
}

@inproceedings{elahi2024fast,
  title     = {Fast Proxy Experiment Design for Causal Effect
               Identification},
  author    = {Elahi, Sepehr and Akbari, Sina and Etesami, Jalal and
               Kiyavash, Negar and Thiran, Patrick},
  booktitle = {Advances in Neural Information Processing Systems (NeurIPS)},
  year      = {2024}
}

@article{kiciman2023causal,
  title   = {Causal Reasoning and Large Language Models: Opening a New
             Frontier for Causality},
  author  = {Kiciman, Emre and Ness, Robert and Sharma, Amit and Tan, Chenhao},
  journal = {Transactions on Machine Learning Research},
  year    = {2023}
}

@article{chi2024unveiling,
  title   = {Unveiling Causal Reasoning in Large Language Models:
             Reality or Mirage?},
  author  = {Chi, Haoang and Li, He and Yang, Wenjing and Liu, Feng and
             Lan, Long and Ren, Xiaoguang and Liu, Tongliang and Han, Bo},
  journal = {Advances in Neural Information Processing Systems},
  volume  = {37},
  pages   = {96640--96670},
  year    = {2024}
}

@article{chen2025reasoning,
  title   = {Reasoning Models Don't Always Say What They Think},
  author  = {Chen, Yanda and Benton, Joe and Radhakrishnan, Ansh and
             Uesato, Jonathan and Denison, Carson and Schulman, John and
             Somani, Arushi and Hase, Peter and Wagner, Misha and
             Roger, Fabien and others},
  journal = {arXiv preprint arXiv:2505.05410},
  year    = {2025}
}

@article{he2026uncovering,
  title   = {Uncovering Hidden Correctness in {LLM} Causal Reasoning via
             Symbolic Verification},
  author  = {He, Paul and Huang, Yinya and Sachan, Mrinmaya and Jin, Zhijing},
  journal = {arXiv preprint arXiv:2601.21210},
  year    = {2026}
}

@inproceedings{shpitser2006conditional,
  title     = {Identification of Conditional Interventional Distributions},
  author    = {Shpitser, Ilya and Pearl, Judea},
  booktitle = {Proceedings of the 22nd Conference on Uncertainty in
               Artificial Intelligence (UAI)},
  pages     = {437--444},
  year      = {2006}
}

@inproceedings{jaber2019causal,
  title     = {Causal Identification under {M}arkov Equivalence:
               Completeness Results},
  author    = {Jaber, Amin and Zhang, Jiji and Bareinboim, Elias},
  booktitle = {International Conference on Machine Learning (ICML)},
  pages     = {2981--2989},
  year      = {2019}
}

@article{tikka2018enhancing,
  title   = {Enhancing Identification of Causal Effects by Pruning},
  author  = {Tikka, Santtu and Karvanen, Juha},
  journal = {Journal of Machine Learning Research},
  volume  = {18},
  number  = {194},
  pages   = {1--23},
  year    = {2018}
}

@article{jin2023cladder,
  title   = {{CLADDER}: Assessing Causal Reasoning in Language Models},
  author  = {Jin, Zhijing and Chen, Yuen and Leeb, Felix and Gresele, Luigi
             and Kamal, Ojasv and Lyu, Zhiheng and Blin, Kevin and
             Gonzalez Adauto, Fernando and Kleiman-Weiner, Max and
             Sachan, Mrinmaya and others},
  journal = {Advances in Neural Information Processing Systems},
  volume  = {36},
  pages   = {31038--31065},
  year    = {2023}
}

@inproceedings{liu2024llms,
  title     = {Are {LLM}s Capable of Data-Based Statistical and Causal
               Reasoning? {B}enchmarking Advanced Quantitative Reasoning
               with Data},
  author    = {Liu, Xiao and Wu, Zirui and Wu, Xueqing and Lu, Pan and
               Chang, Kai-Wei and Feng, Yansong},
  booktitle = {Findings of the Association for Computational Linguistics:
               ACL 2024},
  pages     = {9215--9235},
  year      = {2024}
}

@inproceedings{zhang2024can,
  title     = {Can {LLM} Graph Reasoning Generalize beyond Pattern
               Memorization?},
  author    = {Zhang, Yizhuo and Wang, Heng and Feng, Shangbin and
               Tan, Zhaoxuan and Han, Xiaochuang and He, Tianxing and
               Tsvetkov, Yulia},
  booktitle = {Findings of the Association for Computational Linguistics:
               EMNLP 2024},
  pages     = {2289--2305},
  year      = {2024}
}

@article{kirichenko2026abstentionbench,
  title   = {{A}bstention{B}ench: Reasoning {LLM}s Fail on Unanswerable
             Questions},
  author  = {Kirichenko, Polina and Ibrahim, Mark and Chaudhuri, Kamalika
             and Bell, Samuel J},
  journal = {Advances in Neural Information Processing Systems},
  volume  = {38},
  year    = {2026}
}

@article{wu2025llm,
  title   = {{LLM} Cannot Discover Causality, and Should Be Restricted to
             Non-Decisional Support in Causal Discovery},
  author  = {Wu, Xingyu and Yu, Kui and Wu, Jibin and Tan, Kay Chen},
  journal = {arXiv preprint arXiv:2506.00844},
  year    = {2025}
}

@article{zecevic2023causal,
  title   = {Causal Parrots: Large Language Models May Talk Causality But
             Are Not Causal},
  author  = {Ze{\v{c}}evi{\'c}, Matej and Willig, Moritz and
             Dhami, Devendra Singh and Kersting, Kristian},
  journal = {Transactions on Machine Learning Research},
  year    = {2023}
}

@article{manski1990nonparametric,
  title   = {Nonparametric Bounds on Treatment Effects},
  author  = {Manski, Charles F},
  journal = {The American Economic Review},
  volume  = {80},
  number  = {2},
  pages   = {319--323},
  year    = {1990}
}

@article{balke1997bounds,
  title   = {Bounds on Treatment Effects from Studies with Imperfect
             Compliance},
  author  = {Balke, Alexander and Pearl, Judea},
  journal = {Journal of the American Statistical Association},
  volume  = {92},
  number  = {439},
  pages   = {1171--1176},
  year    = {1997}
}

@article{cinelli2020making,
  title   = {Making Sense of Sensitivity: Extending Omitted Variable Bias},
  author  = {Cinelli, Carlos and Hazlett, Chad},
  journal = {Journal of the Royal Statistical Society: Series B},
  volume  = {82},
  number  = {1},
  pages   = {39--67},
  year    = {2020}
}

@article{vanderweele2017sensitivity,
  title   = {Sensitivity Analysis in Observational Research:
             Introducing the {E}-Value},
  author  = {VanderWeele, Tyler J and Ding, Peng},
  journal = {Annals of Internal Medicine},
  volume  = {167},
  number  = {4},
  pages   = {268--274},
  year    = {2017}
}

@book{rosenbaum2002observational,
  title     = {Observational Studies},
  author    = {Rosenbaum, Paul R},
  edition   = {2nd},
  publisher = {Springer},
  year      = {2002}
}

@inproceedings{wang2024causalbench,
  title     = {{C}ausal{B}ench: A Comprehensive Benchmark for Evaluating Causal
               Reasoning Capabilities of Large Language Models},
  author    = {Wang, Zeyu},
  booktitle = {Proceedings of the 10th SIGHAN Workshop on Chinese Language
               Processing (SIGHAN-10)},
  pages     = {143--151},
  year      = {2024}
}

@inproceedings{xu2026noisycausal,
  title     = {{N}oisy{C}ausal: A Benchmark for Evaluating Causal Reasoning
               Under Structured Noise},
  author    = {Xu, Zhi and Fu, Yun},
  booktitle = {Proceedings of the 64th Annual Meeting of the Association for
               Computational Linguistics (Volume 1: Long Papers)},
  pages     = {39500--39513},
  year      = {2026}
}

@article{lee2026econcausal,
  title   = {{E}con{C}ausal: A Context-Aware Economic Reasoning Benchmark for
             Large Language Models},
  author  = {Lee, Donggyu and Yun, Hyeok and Cha, Meeyoung and Park, Sungwon
             and Park, Sangyoon and Kim, Jihee},
  journal = {arXiv preprint arXiv:2510.07231},
  year    = {2026}
}

@article{wang2026case,
  title   = {{CASE}: Causal Alignment and Structural Enforcement for Improving
             Chain-of-Thought Faithfulness},
  author  = {Wang, Ziming and Yao, Yinghua and Huang, Changwu and Tang, Ke and
             Yao, Xin},
  journal = {arXiv preprint arXiv:2607.18820},
  year    = {2026}
}

@inproceedings{lanham2023measuring,
  title   = {Measuring Faithfulness in Chain-of-Thought Reasoning},
  author  = {Lanham, Tamera and Chen, Anna and Radhakrishnan, Ansh and
             Steiner, Benoit and Denison, Carson and Hernandez, Danny and
             Li, Dustin and Durmus, Esin and Hubinger, Evan and Kernion, Jackson
             and others},
  journal = {arXiv preprint arXiv:2307.13702},
  year    = {2023}
}

@article{turpin2023language,
  title   = {Language Models Don't Always Say What They Think: Unfaithful
             Explanations in Chain-of-Thought Prompting},
  author  = {Turpin, Miles and Michael, Julian and Perez, Ethan and
             Bowman, Samuel R},
  journal = {Advances in Neural Information Processing Systems},
  volume  = {36},
  pages   = {74952--74965},
  year    = {2023}
}

@article{sawarni2026causalreasoning,
  title   = {{C}ausal{R}easoning{B}enchmark: A Real-World Benchmark for
             Disentangled Evaluation of Causal Identification and
             Estimation},
  author  = {Sawarni, Ayush and Tan, Jiyuan and Syrgkanis, Vasilis},
  journal = {arXiv preprint arXiv:2602.20571},
  year    = {2026}
}

@article{yang2024critical,
  title   = {A Critical Review of Causal Reasoning Benchmarks for Large
             Language Models},
  author  = {Yang, Linying and Wang, Vik Shirvaikar and Alaa, Ahmed and
             Blei, David},
  journal = {arXiv preprint arXiv:2407.08029},
  year    = {2024}
}

@article{robins2000marginal,
  title={Marginal structural models and causal inference in epidemiology},
  author={Robins, James M and Hernan, Miguel Angel and Brumback, Babette},
  journal={Epidemiology},
  volume={11},
  number={5},
  pages={550--560},
  year={2000}
}

@article{ryalen2026causal,
  title={On causal inference with marked point process data},
  author={Ryalen, P{\aa}l Christie and Stensrud, Mats Julius and R{\o}ysland, Kjetil},
  journal={arXiv preprint arXiv:2604.12977},
  year={2026}
}

@article{tennant2021use,
  title={Use of directed acyclic graphs (DAGs) to identify confounders in applied health research: review and recommendations},
  author={Tennant, Peter WG and Murray, Eleanor J and Arnold, Kellyn F and Berrie, Laurie and Fox, Matthew P and Gadd, Sarah C and Harrison, Wendy J and Keeble, Claire and Ranker, Lynsie R and Textor, Johannes and others},
  journal={International journal of epidemiology},
  volume={50},
  number={2},
  pages={620--632},
  year={2021},
  publisher={Oxford University Press}
}

@article{binz2023using,
  title={Using cognitive psychology to understand GPT-3},
  author={Binz, Marcel and Schulz, Eric},
  journal={Proceedings of the National Academy of Sciences},
  volume={120},
  number={6},
  pages={e2218523120},
  year={2023},
  publisher={National Academy of Sciences}
}

@article{hagendorff2023machine,
  title={Machine psychology},
  author={Hagendorff, Thilo and Dasgupta, Ishita and Binz, Marcel and Chan, Stephanie CY and Lampinen, Andrew and Wang, Jane X and Akata, Zeynep and Schulz, Eric},
  journal={arXiv preprint arXiv:2303.13988},
  year={2023}
}

@article{coda2024cogbench,
  title={CogBench: a large language model walks into a psychology lab},
  author={Coda-Forno, Julian and Binz, Marcel and Wang, Jane X and Schulz, Eric},
  journal={arXiv preprint arXiv:2402.18225},
  year={2024}
}

@article{parasuraman1997humans,
  title={Humans and automation: Use, misuse, disuse, abuse},
  author={Parasuraman, Raja and Riley, Victor},
  journal={Human factors},
  volume={39},
  number={2},
  pages={230--253},
  year={1997},
  publisher={SAGE Publications Sage CA: Los Angeles, CA}
}

@article{buccinca2021trust,
  title={To trust or to think: cognitive forcing functions can reduce overreliance on AI in AI-assisted decision-making},
  author={Bu{\c{c}}inca, Zana and Malaya, Maja Barbara and Gajos, Krzysztof Z},
  journal={Proceedings of the ACM on Human-computer Interaction},
  volume={5},
  number={CSCW1},
  pages={1--21},
  year={2021},
  publisher={ACM New York, NY, USA}
}

@inproceedings{schemmer2023appropriate,
  title={Appropriate reliance on AI advice: Conceptualization and the effect of explanations},
  author={Schemmer, Max and Kuehl, Niklas and Benz, Carina and Bartos, Andrea and Satzger, Gerhard},
  booktitle={Proceedings of the 28th international conference on intelligent user interfaces},
  pages={410--422},
  year={2023}
}

@article{kadavath2022language,
  title={Language models (mostly) know what they know},
  author={Kadavath, Saurav and Conerly, Tom and Askell, Amanda and Henighan, Tom and Drain, Dawn and Perez, Ethan and Schiefer, Nicholas and Hatfield-Dodds, Zac and DasSarma, Nova and Tran-Johnson, Eli and others},
  journal={arXiv preprint arXiv:2207.05221},
  year={2022}
}

@article{lambert20243,
  title={3: Pushing frontiers in open language model post-training. CoRR, abs/2411.15124, 2024. doi: 10.48550},
  author={Lambert, N and Morrison, J and Pyatkin, V and Huang, S and Ivison, H and Brahman, F and Miranda, LJV and Liu, A and Dziri, N and Lyu, S and others},
  journal={arXiv preprint ARXIV.2411.15124},
  pages={9},
  year={2024}
}

@article{shao2024deepseekmath,
  title={Deepseekmath: Pushing the limits of mathematical reasoning in open language models},
  author={Shao, Zhihong and Wang, Peiyi and Zhu, Qihao and Xu, Runxin and Song, Junxiao and Bi, Xiao and Zhang, Haowei and Zhang, Mingchuan and Li, YK and Wu, Yang and others},
  journal={arXiv preprint arXiv:2402.03300},
  year={2024}
}

@article{guo2025deepseek,
  title={Deepseek-r1: Incentivizing reasoning capability in llms via reinforcement learning},
  author={Guo, Daya and Yang, Dejian and Zhang, Haowei and Song, Junxiao and Wang, Peiyi and Zhu, Qihao and Xu, Runxin and Zhang, Ruoyu and Ma, Shirong and Bi, Xiao and others},
  journal={arXiv preprint arXiv:2501.12948},
  year={2025}
}

@article{skalse2022defining,
  title={Defining and characterizing reward gaming},
  author={Skalse, Joar and Howe, Nikolaus and Krasheninnikov, Dmitrii and Krueger, David},
  journal={Advances in neural information processing systems},
  volume={35},
  pages={9460--9471},
  year={2022}
}

@incollection{verma2022equivalence,
  title={Equivalence and synthesis of causal models},
  author={Verma, Thomas S and Pearl, Judea},
  booktitle={Probabilistic and causal inference: The works of Judea Pearl},
  pages={221--236},
  year={2022}
}

@article{schwartz1980fast,
  title={Fast probabilistic algorithms for verification of polynomial identities},
  author={Schwartz, Jacob T},
  journal={Journal of the ACM (JACM)},
  volume={27},
  number={4},
  pages={701--717},
  year={1980},
  publisher={ACM New York, NY, USA}
}

@inproceedings{zippel1979probabilistic,
  title={Probabilistic algorithms for sparse polynomials},
  author={Zippel, Richard},
  booktitle={International symposium on symbolic and algebraic manipulation},
  pages={216--226},
  year={1979},
  organization={Springer}
}

@article{binder1997adaptive,
  title={Adaptive probabilistic networks with hidden variables},
  author={Binder, John and Koller, Daphne and Russell, Stuart and Kanazawa, Keiji},
  journal={Machine Learning},
  volume={29},
  number={2},
  pages={213--244},
  year={1997},
  publisher={Springer}
}

@inproceedings{beinlich1989alarm,
  title={The ALARM monitoring system: A case study with two probabilistic inference techniques for belief networks},
  author={Beinlich, Ingo A and Suermondt, Henri Jacques and Chavez, R Martin and Cooper, Gregory F},
  booktitle={AIME 89: Second European Conference on Artificial Intelligence in Medicine, London, August 29th--31st 1989. Proceedings},
  pages={247--256},
  year={1989},
  organization={Springer}
}

@inproceedings{onisko2003probabilistic,
  title={Probabilistic causal models in medicine: Application to diagnosis of liver disorders},
  author={Onisko, Agnieszka},
  booktitle={Ph. D. dissertation, Inst. Biocybern. Biomed. Eng., Polish Academy Sci., Warsaw, Poland},
  year={2003}
}

@article{spiegelhalter1992learning,
  title={Learning in probabilistic expert systems},
  author={Spiegelhalter, David J},
  journal={Bayesian statistics},
  volume={4},
  pages={447--465},
  year={1992},
  publisher={Oxford University Press}
}

@article{lauritzen1988local,
  title={Local computations with probabilities on graphical structures and their application to expert systems},
  author={Lauritzen, Steffen L and Spiegelhalter, David J},
  journal={Journal of the Royal Statistical Society: Series B (Methodological)},
  volume={50},
  number={2},
  pages={157--194},
  year={1988},
  publisher={Wiley Online Library}
}

@inproceedings{behnam2024graph,
  title={Graph neural network causal explanation via neural causal models},
  author={Behnam, Arman and Wang, Binghui},
  booktitle={European conference on computer vision},
  pages={410--427},
  year={2024},
  organization={Springer}
}

@article{behnam2026measure,
  title={Measure-Theoretic Anti-Causal Representation Learning},
  author={Behnam, Arman and Wang, Binghui},
  journal={Advances in Neural Information Processing Systems},
  volume={38},
  pages={61375--61431},
  year={2025}
}

@inproceedings{behnam2024causal,
  title={Causal Explanation from Mild Cognitive Impairment Progression using Graph Neural Networks},
  author={Behnam, Arman and Garg, Muskan and Liu, Xingyi and Vassilaki, Maria and {St. Sauver}, Jennifer and Petersen, Ronald C and Sohn, Sunghwan},
  booktitle={2024 IEEE International Conference on Bioinformatics and Biomedicine (BIBM)},
  pages={6349--6355},
  year={2024},
  organization={IEEE}
}

@inproceedings{behnam2026reasoning,
  title     = {Structure-agnostic Causal Representation Learning},
  author    = {Behnam, Arman and Wang, Binghui},
  booktitle = {Advances in Neural Information Processing Systems},
  year      = {2026},
  url       = {https://openreview.net/forum?id=8gtLRyYY45}
}

@article{scutari2010learning,
  title={Learning Bayesian networks with the bnlearn R package},
  author={Scutari, Marco},
  journal={Journal of statistical software},
  volume={35},
  number={1},
  pages={1--22},
  year={2010}
}

\clearpage
\part*{Appendix}
\appendix
\label{sec:appendix}

\startcontents[appendix]
\printcontents[appendix]{}{1}{
\setcounter{tocdepth}{2}}

\clearpage
\section{Notation and Formal Definitions}
\label{app:formal}
 
This appendix supports the body in two ways. Appendix~\ref{app:notation} collects every symbol, so a reader can check a definition without searching Sections~\ref{sec:background}--\ref{sec:benchmark}. Appendix~\ref{app:hedges} states the graphical characterization of non-identifiability. The body uses it only through \ID{}, except in Theorem~\ref{thm:destroy} and its proof. The hedge definitions are prior work.

\subsection{Notation}
\label{app:notation}
 
Table~\ref{tab:notation} collects every symbol used in the paper. Uppercase denotes random variables, lowercase their realizations, and boldface sets of variables. Graph-theoretic notation follows \citet{pearl2009causality}, and identification notation follows \citet{shpitser2006identification}.
 
\begin{table}[h]
\centering\small
\caption{Symbols used in the paper, with the section in which each is
introduced.}
\label{tab:notation}
\begin{tabular}{lll}
\toprule
Symbol & Meaning & Introduced \\
\midrule
\multicolumn{3}{l}{\emph{graphs and models}} \\
$\Gr$ & acyclic directed mixed graph over $\Vs$ & Sec.~\ref{sec:background} \\
$\Vs$ & observed variables of $\Gr$ & Sec.~\ref{sec:background} \\
$\Xs, \Ys$ & disjoint treatment and outcome sets & Sec.~\ref{sec:background} \\
$\Bi{\Gr}$ & bi-directed edges of $\Gr$ & Sec.~\ref{sec:background} \\
$\An(\Ys)$ & ancestors of $\Ys$ & Sec.~\ref{sec:background} \\
$\Gsub{\Xs}$ & $\Gr$ with edges into $\Xs$ removed & App.~\ref{app:hedges} \\
$\Scm$, $\Mods{\Gr}$ & an SCM; the SCMs whose dependencies lie in $\Gr$ & Sec.~\ref{sec:background} \\
$P(\Vs)$ & observational distribution & Sec.~\ref{sec:background} \\
$\theta$ & parameters of a discrete SCM & Thm.~\ref{thm:grading} \\
$F^{\leftrightarrow}$ & bi-directed edges of a subgraph $F$ & App.~\ref{app:proofs} \\
\midrule
\multicolumn{3}{l}{\emph{identification}} \\
$\Qy = \Px{\mathbf{x}}{\mathbf{y}}$ & query $P(\mathbf{y} \mid \Do{\mathbf{x}})$ & Sec.~\ref{sec:background} \\
$\ID(\Gr,\Qy)$ & identification algorithm; returns $\est$ or $\FAIL$ & Sec.~\ref{sec:background} \\
$\est$ & a closed-form estimand & Sec.~\ref{sec:background} \\
$\Rs$ & root set of a C-forest & App.~\ref{app:hedges} \\
$\mathcal{H}(\Gr,\Qy)$ & hedges for $\Qy$ in $\Gr$ over all $\Xs'\!\subseteq\!\Xs$, $\Ys'\!\subseteq\!\Ys$ & Thm.~\ref{thm:destroy} \\
\midrule
\multicolumn{3}{l}{\emph{problem and evaluation}} \\
$\ell$ & certified label $\ID(\Gr,\Qy)$ & Def.~\ref{def:certified} \\
$r = (v, \hat\est)$ & response: verdict and returned estimand & Def.~\ref{def:evalproblem} \\
$k$, $\tau$ & SCM draws per instance; numerical tolerance & Def.~\ref{def:verified} \\
$\mathcal{S}, \mathcal{I}, \mathcal{N}$ & scored, identifiable and non-identifiable instances & Sec.~\ref{sec:metrics} \\
\midrule
\multicolumn{3}{l}{\emph{construction and repair}} \\
$p_{\mathrm{dir}}, p_{\mathrm{bi}}$ & directed and bi-directed edge probabilities & Def.~\ref{def:random-family} \\
$D$, $L$ & a published DAG and the latent set hidden from it & Def.~\ref{def:semi-family} \\
$\Aug \subseteq \Bi{\Gr}$ & augmentation: bi-directed edges assumed absent & Sec.~\ref{sec:repair-labels} \\
$\Gr \gminus \Aug$ & $\Gr$ with the edges of $\Aug$ deleted & Sec.~\ref{sec:repair-labels} \\
$\Adm(\Gr,\Qy)$ & augmentations under which $\Qy$ becomes identifiable & Sec.~\ref{sec:repair-labels} \\
$\Frontier(\Gr,\Qy)$ & $\subseteq$-minimal elements of $\Adm(\Gr,\Qy)$ & Sec.~\ref{sec:repair-labels} \\
\bottomrule
\end{tabular}
\end{table}
 
\subsection{C-components, C-forests and hedges}
\label{app:hedges}
 
The body needs hedges only for Theorem~\ref{thm:destroy}, which states when an augmentation destroys a hedge. Its proof uses the fact that most conditions in the definition below survive the deletion of a bi-directed edge. The three definitions and the criterion are stated here in the form given by~\citet{shpitser2006identification}.
 
\begin{definition}[C-component; \citealp{tian2002general}] A set $\mathbf{C} \subseteq \Vs$ is a \emph{C-component} of $\Gr$ if every pair of vertices in $\mathbf{C}$ is connected by a path consisting entirely of bi-directed edges.
\end{definition}
 
\begin{definition}[C-forest; \citealp{shpitser2006identification}]
Let $F$ be a subgraph of $\Gr$, and let its root set $\Rs$ be the vertices of $F$ that have no children in $F$. $F$ is an \emph{$\Rs$-rooted C-forest} if the bi-directed edges of $F$ connect all its vertices and every vertex of $F$ has at most one child in $F$.
\end{definition}
 
\begin{definition}[Hedge; \citealp{shpitser2006identification}]
\label{def:hedge}
Let $F$ and $F'$ be $\Rs$-rooted C-forests with $F' \subseteq F$, $F \cap \Xs \neq \emptyset$, $F' \cap \Xs = \emptyset$, and $F \subseteq \An(\Ys)_{\GX}$. Then $F, F'$ form a \emph{hedge} for $\Px{\mathbf{x}}{\mathbf{y}}$ in $\Gr$.
\end{definition}
 
\begin{priorresult}
$\Px{\mathbf{x}}{\mathbf{y}}$ is identifiable from $P(\Vs)$ in $\Gr$ if and only if there is no hedge for $\Px{\mathbf{x}'}{\mathbf{y}'}$ in $\Gr$ for any $\Xs' \subseteq \Xs$ and $\Ys' \subseteq \Ys$~\citep{shpitser2006identification}.
\end{priorresult}
 
The criterion quantifies over all subsets of $\Xs$ and $\Ys$, so a single graph may contain many hedges and enumerating them is expensive. Theorem~\ref{thm:destroy} therefore takes the hedges as given.
\section{Extended Related Work}
\label{app:related}

This appendix places our work in five literatures: identification given a graph, what to do when identification fails, benchmarks for causal reasoning, chain-of-thought faithfulness, and studies of causal reasoning in language models.

\subsection{Identification given a graph}

Given a causal graph, the question of which interventional quantities are recoverable from observational data is settled. Do-calculus provides the rewrite rules~\citep{pearl1995causal} and is complete for interventional identification~\citep{huang2006pearl}. \ID{} decides the question constructively and returns a hedge on failure~\citep{tian2002general, shpitser2006identification, shpitser2008complete}, with extensions to conditional effects~\citep{shpitser2006conditional} and to Markov equivalence classes~\citep{jaber2019causal}. Returned expressions can be simplified and their graphs pruned~\citep{tikka2017simplifying, tikka2018enhancing}, and valid adjustment sets are completely characterized~\citep{perkovic2018complete}. We use these results without modification. The completeness of \ID{} makes our labels theorems, and its polynomial running time lets us generate instances at arbitrary scale.

\subsection{When identification fails}
\label{sec:related-fail}

Two responses exist. The first repairs the graph. \citet{akbari2022minimum} pose minimum-cost intervention design, show it to be NP-hard, and give an exact algorithm and a logarithmic-factor approximation by casting it as a minimum hitting set over hedges: an intervention set restores identifiability exactly when it hits every hedge. \citet{elahi2024fast} show the number of minimal hedges can be exponential and recast the problem as weighted MAX-SAT, an integer program, and submodular maximization. \citet{akbari2023experimental} extend the problem to experimental design. The second response accepts the failure and quantifies its consequences: partial identification returns bounds~\citep{manski1990nonparametric, balke1997bounds}, and sensitivity analysis describes how conclusions erode as an assumption weakens~\citep{rosenbaum2002observational, cinelli2020making, vanderweele2017sensitivity}. Our repair task differs from the first response in what it changes. Deleting a bi-directed edge asserts that a confounder is absent, while an intervention removes the confounder's influence by setting a variable directly. Theorem~\ref{thm:destroy} shows that this changes what restores identifiability: an assumption set must cut the bi-directed edges of every hedge, where an intervention set must hit its vertices. We do not study the resulting optimization problem; we use the characterization to score repair responses exactly.

\subsection{Benchmarks for causal reasoning}
\label{sec:related-bench}

\textsc{cladder} verbalizes symbolic causal problems across the three rungs of the ladder of causation and solves each with an inference engine to obtain ground truth~\citep{jin2023cladder}. \textsc{qrdata} pairs causal questions with real data tables~\citep{liu2024llms}; \textsc{causalbench} probes each scenario from several directions to reduce success by chance~\citep{wang2024causalbench}, and \textsc{noisycausal} generates instances from ground-truth graphs and injects structured noise, including latent confounders~\citep{xu2026noisycausal}. In each case the ground truth is a value computed from a fully specified model, and every question posed is answerable by construction. \textsc{econcausal} evaluates context-dependent causal claims drawn from empirical economics and reports a finding adjacent to ours:  offered an explicit \emph{Unknown} option with no identifying context, models select that option under $10\%$ of the time and commit to a directional sign~\citep{lee2026econcausal}. Their unanswerable cases arise from underspecification, while ours are non-identifiable as a theorem. That difference lets us measure a false-claim rate against ground truth, where they measure a tendency to over-commit. On grading, \citet{he2026uncovering} address the same problem by checking symbolically whether a generated expression is derivable from the graph under do-calculus. We grade numerically instead, by evaluating the expression on an SCM whose interventional distribution is known exactly. Numerical grading accepts a correct expression written in any form. \textsc{abstentionbench} evaluates when models should decline to answer~\citep{kirichenko2026abstentionbench}, with unanswerable instances fixed by annotation where ours are decided by a complete algorithm.

Two recent lines converge on our motivation from different directions. \citet{yang2024critical} survey causal reasoning benchmarks and observe that many can be solved by retrieving domain knowledge, which questions whether those benchmarks measure what they claim. The random family of Sec.~\ref{sec:families} removes that route by construction. Closer still, \textsc{crb}~\citep{sawarni2026causalreasoning} argues that scoring a single numerical output conflates identification with estimation, and separates the two by requiring a structured research design alongside a point estimate on real datasets. We share the diagnosis and differ in two respects: what identification means, and where the ground truth comes from. Their specifications name a design (instrumental variables, difference-in-differences, regression discontinuity) and a control set, curated from published papers. Ours is graphical identifiability, decided by \ID{} and graded by numerical evaluation. Their labels are curated; ours are computed. The two are complementary, and neither subsumes the other.

\subsection{Chain-of-thought faithfulness}
\label{sec:related-cot}

A parallel literature asks whether a generated reasoning trace reflects the computation that produced the answer. Models fail to verbalize cues that demonstrably determined their output~\citep{turpin2023language, chen2025reasoning}. Faithfulness is measured by perturbing the trace and observing whether the answer changes~\citep{lanham2023measuring}, and methods such as \textsc{case} intervene on training and inference to force the answer through the trace~\citep{wang2026case}. All of this presupposes that a trace exists. Sec.~\ref{sec:res-warrant} reports that on this task a trace frequently does not exist. Under a default prompt, two of three frontier models return every correct refusal as a bare verdict, so the question of faithfulness does not arise until a trace is requested.

\subsection{Language models and causal reasoning}

Behavioral studies attribute measured causal competence to knowledge learned from text, attached to variable names, and find little evidence of structural analysis~\citep{kiciman2023causal, zecevic2023causal}. \citet{chi2024unveiling} report sharp degradation on corpora published after the training cutoff. \citet{zhang2024can} find that graph reasoning fails to transfer under controlled distribution shift, and \citet{wu2025llm} argue that language models should be confined to non-decisional roles in causal discovery. Sec.~\ref{sec:res-verdict} reports results that sit uneasily with the strongest form of these claims, measured on a task where neither variable names nor memorized structure is available.
\section{Properties of the Certification Procedure}
\label{app:properties}

This appendix builds on three results from the body. Identifiability is monotone under edge deletion (Proposition~\ref{prop:mono}). Grading refutes a wrong estimand almost surely (Theorem~\ref{thm:grading}). An augmentation restores identifiability exactly when the augmentation destroys every hedge (Theorem~\ref{thm:destroy}). This appendix records what follows from them. The consequences answer four questions: how large the frontier can be and which edges it uses, how expensive membership is, how much of the grading guarantee survives our implementation's approximations, and what the reported accuracy and leakage license. Nothing here is required to follow the body, and all proofs are deferred to Appendix~\ref{app:proofs}.

\subsection{Structure and cost of the restoring set}
\label{app:properties-frontier}

Proposition~\ref{prop:mono} makes $\Adm(\Gr,\Qy)$ an up-set, so its minimal elements cannot contain one another.

\begin{corollary}[Antichain bound]
\label{cor:antichain}
$\Frontier(\Gr,\Qy)$ is an anti-chain in the Boolean lattice on $\Bi{\Gr}$, so with $m = |\Bi{\Gr}|$,
\[   |\Frontier(\Gr,\Qy)| \;\le\; \binom{m}{\lfloor m/2 \rfloor}.\]
\end{corollary}

The bound is exponential in $m$ and is attained by the family of all $\lfloor m/2 \rfloor$-subsets, so nothing about the lattice structure prevents the frontier from being large. Exhaustive enumeration becomes infeasible in Table~\ref{tab:repairscale} for a different reason: it searches all $2^m$ subsets of $\Bi{\Gr}$, and $m$ grows with the graph. The oracle-guided search queries only the subsets it needs. Enumeration pays for the whole anti-chain; a search that queries only the subsets it needs does not. Relatedly, \citet{elahi2024fast} exhibit graphs whose minimal hedges are exponentially many.

Theorem~\ref{thm:destroy} reduces membership to a connectivity question once the hedges are known.

\begin{corollary}[Membership without calling \ID{}]
\label{cor:membership}
Given $\mathcal{H}(\Gr,\Qy)$, deciding $\Aug \in \Adm(\Gr,\Qy)$ costs
$O\!\left(|\mathcal{H}| \cdot (|\Vs| + |\Bi{\Gr}|)\right)$ and no call to \ID{}.
\end{corollary}

We do not use this corollary in our evaluation, because obtaining $\mathcal{H}(\Gr,\Qy)$ is itself the expensive step: \ID{} returns one witness, not all hedges. The corollary matters for a different reason. It locates where the cost of re[Membership without an oracle call]pair actually sits. Deciding whether a proposed augmentation works is cheap given the obstructions; finding the obstructions is not.

\begin{remark}[Covering by cuts, not by elements]
\label{rem:cuts}
By Theorem~\ref{thm:destroy}, $\Aug \in \Adm(\Gr,\Qy)$ exactly when $\Aug$ contains a disconnecting set of $F^{\leftrightarrow}$ or of $F'^{\leftrightarrow}$ for every $(F,F') \in \mathcal{H}(\Gr,\Qy)$. The covering objects are therefore edge cuts, not single edges. That distinguishes our formulation from the intervention design of \citet{akbari2022minimum}, where an intervention set restores identifiability exactly when the set \emph{hits} every hedge.
\end{remark}

In one case, the edge formulation also becomes a hitting-set problem.

\begin{corollary}[Reduction to hitting set in the tree case]
\label{cor:hardness}
If the bidirected structure of every hedge in $\mathcal{H}(\Gr,\Qy)$ is a tree, then $\Aug \in \Adm(\Gr,\Qy)$ exactly when $\Aug$ meets $F^{\leftrightarrow}$ for every $(F,F') \in \mathcal{H}(\Gr,\Qy)$. Finding a minimum-cardinality restoring augmentation is then an instance of minimum hitting set over the family $\{F^{\leftrightarrow}\}$.
\end{corollary}

We state this as a reduction and not as a hardness result, and the distinction matters. Corollary~\ref{cor:hardness} exhibits our problem as an instance of hitting set, which bounds it from above and says nothing about its difficulty. Establishing NP-hardness requires the opposite direction: a construction sending an arbitrary hitting-set instance to a graph whose hedge family realizes it. \citet{akbari2022minimum} give such a construction for the vertex formulation, where an intervention set restores identifiability exactly when it hits every hedge. We have not carried it over to edge deletion and make no hardness claim for the edge formulation.

Two further consequences of Theorem~\ref{thm:destroy} locate the edges a minimal repair can use and bound its size.

\begin{corollary}[Minimal repairs lie inside $\An(\Ys)$]
\label{cor:ancestral}
No element of $\Frontier(\Gr,\Qy)$ contains a bi-directed edge with an endpoint outside $\An(\Ys)$.
\end{corollary}

Enlarging the graph outside $\An(\Ys)$ therefore adds no candidate edge to any minimal repair.

For a subgraph $F$, let $\lambda(F)$ be the least number of edges of $F^{\leftrightarrow}$ whose deletion leaves $V(F)$ disconnected, with $\lambda(F) = \infty$ when $F$ has a single vertex.

\begin{corollary}[Lower bound on repair size]
\label{cor:lowerbound}
Every $\Aug \in \Adm(\Gr,\Qy)$ satisfies $|\Aug| \ge \min\{\lambda(F), \lambda(F')\}$ for each $(F,F') \in \mathcal{H}(\Gr,\Qy)$. Over a family of hedges whose sets $F^{\leftrightarrow}$ are pairwise disjoint, the bounds add.
\end{corollary}

The bound needs no enumeration. Each $\lambda$ is a minimum edge cut, computable in polynomial time, and the bound holds for any known hedge, including the one \ID{} returns on failure.

\subsection{Scope of the grading guarantee}
\label{app:properties-grading}

Theorem~\ref{thm:grading} is stated with $\tau = 0$ and exact arithmetic. Our implementation uses neither. Three consequences delimit what survives.

\begin{corollary}[One draw suffices in exact arithmetic]
\label{cor:onedraw}
Under the hypotheses of Theorem~\ref{thm:grading} with $\tau = 0$, an estimand that disagrees with $\Qy$ on the parametrized family is refuted by a single draw with probability one. Drawing $k > 1$ SCMs therefore adds nothing to the guarantee.
\end{corollary}

In exact arithmetic, one draw therefore suffices. We draw $k = \nSCMDraws{}$ because of floating-point arithmetic: with $\tau > 0$, a wrong estimand can agree with the target within tolerance on one draw by coincidence, and a second draw makes that unlikely. The second draw adds robustness to finite precision, not certainty about the estimand.

\begin{remark}[The guarantee degrades continuously in $\tau$]
\label{rem:tau}
Let $D(\theta) = \hat\est(P_\theta(\Vs)) - \Qy_\theta$. The sets $\{\theta : |D(\theta)| \le \tau\}$ decrease as $\tau \downarrow 0$ and intersect in $\{\theta : D(\theta) = 0\}$, which has measure zero by Theorem~\ref{thm:grading}. Their measure therefore tends to zero with $\tau$: a wrong estimand escapes refutation on a positive but vanishing fraction of parameter space.
\end{remark}

\begin{remark}[The comparison certifies a slice]
\label{rem:slice}
Our implementation compares at one intervention value (Sec.~\ref{sec:verifier}), so Theorem~\ref{thm:grading} certifies agreement on $P(\mathbf{y} \mid \Do{x})$ at that value. Certifying the full interventional distribution requires the comparison at each value in the domain of $\Xs$, multiplying the cost of the enumeration by $|\mathrm{dom}(\Xs)|$. That strengthening is the obvious one, and we did not run it.
\end{remark}

\begin{remark}[A refutation is a certificate]
\label{rem:certificate}
A refutation exhibits a parameter $\theta$ with $|\hat\est(P_\theta(\Vs)) - \Qy_\theta| > \tau$. Given $\theta$, anyone can recompute both quantities, in exact rational arithmetic if desired, and confirm the refutation without trusting our implementation. Verification has no such certificate; its strength is Theorem~\ref{thm:grading}(ii).
\end{remark}

\subsection{What the reported numbers license}
\label{app:properties-metrics}

\begin{corollary}[Accuracy does not determine the false-claim rate]
\label{cor:acc-fc}
Let a pool have $|\mathcal{I}| = |\mathcal{N}|$. Let a system have false-claim rate $f$ and miss rate $m$, the fraction of $\mathcal{I}$ answered with a miss. Its accuracy is $a = 1 - (f + m)/2$. Hence, for $a \ge 1/2$, the false-claim rate satisfies $0 \le f \le 2(1 - a)$, and both bounds are attained.
\end{corollary}

At $a = 0.98$, the false-claim rate can lie anywhere in $[0, 0.04]$; at $a = 0.90$, anywhere in $[0, 0.20]$. Accuracy therefore cannot rank systems by soundness, and Sec.~\ref{sec:res-decouple} observes this spread in practice.

\begin{remark}[Reported leakage bounds a classifier, not the graph]
\label{rem:leak-bound}
Leakage (Definition~\ref{def:leakage}) is the AUC of one classifier. A better classifier on graph features alone may separate real from spurious edges further. Exceeding the reported leakage is therefore necessary, and not sufficient, for a system to be credited with using the variable names.
\end{remark}
\section{Proofs}
\label{app:proofs}

This appendix proves every result stated in the body and in Appendix~\ref{app:properties}, in the order they appear. Theorem~\ref{thm:destroy} and its consequences use the definitions of Appendix~\ref{app:hedges}.

\subsection{Monotonicity and minimization}

\begin{proof}[Proof of Proposition~\ref{prop:mono}]
Let $\Gr'$ be $\Gr$ with any set of edges deleted. By definition, $\Mods{\Gr}$ contains every SCM whose structural dependencies are contained in $\Gr$. Every dependency permitted by $\Gr'$ is permitted by $\Gr$, so $\Mods{\Gr'} \subseteq \Mods{\Gr}$. If $\Qy$ is identifiable in $\Gr$ there is a functional $f$ with $\Qy_\Scm = f(P_\Scm(\Vs))$ for every $\Scm \in \Mods{\Gr}$; the same $f$ serves every $\Scm \in \Mods{\Gr'}$, so $\Qy$ is identifiable in $\Gr'$.

For upward closure, let $\Aug \in \Adm(\Gr,\Qy)$ and $\Aug \subseteq \Aug'$. Writing $\Gr \gminus \Aug' = (\Gr \gminus \Aug) \gminus (\Aug' \setminus \Aug)$ and applying the above to $\Gr \gminus \Aug$ gives $\Aug' \in \Adm(\Gr,\Qy)$.

For non-emptiness, $\Gr \gminus \Bi{\Gr}$ has no bi-directed edges, so every SCM it represents is Markovian and every interventional distribution is identifiable by truncated factorization~\citep{pearl2009causality}. Hence $\Bi{\Gr} \in \Adm(\Gr,\Qy)$, the set is non-empty, and since it is a finite up-set it is the upward closure of its minimal elements $\Frontier(\Gr,\Qy)$.
\end{proof}

\begin{proof}[Proof of Lemma~\ref{lem:shrink}]
Write $\Aug_0 = \Aug$ and let $\Aug_{i}$ be the set after the $i$th element has been considered, so $\Aug_{i}$ is $\Aug_{i-1}$ with that element deleted when the result lies in $\Adm(\Gr,\Qy)$ and equal to $\Aug_{i-1}$ otherwise. Every $\Aug_i$ lies in $\Adm(\Gr,\Qy)$ by construction, so the result $\Aug^\star = \Aug_{|\Aug|}$ does.

Suppose $\Aug^\star \notin \Frontier(\Gr,\Qy)$. Then some $\Aug'' \subsetneq \Aug^\star$ lies in $\Adm(\Gr,\Qy)$; choose $e \in \Aug^\star \setminus \Aug''$. Since $\Aug'' \subseteq \Aug^\star \setminus \{e\}$, upward closure gives $\Aug^\star \setminus \{e\} \in \Adm(\Gr,\Qy)$. Let $i$ be the step at which $e$ was considered. Because elements are only ever removed, $\Aug^\star \subseteq \Aug_{i-1}$, so $\Aug^\star \setminus \{e\} \subseteq \Aug_{i-1} \setminus \{e\}$ and upward closure gives $\Aug_{i-1} \setminus \{e\} \in \Adm(\Gr,\Qy)$. The rule would then have deleted $e$, contradicting $e \in \Aug^\star$. Each element is tested once, so the procedure makes exactly $|\Aug|$ calls to \ID{}.
\end{proof}

\subsection{Hedge destruction and its consequences}

Theorem~\ref{thm:destroy} rests on the fact that deleting bi-directed edges cannot create a hedge, which we state separately because both directions of the theorem use it.

\begin{lemma}[Hedge persistence]
\label{lem:persist}
Let $\Aug \subseteq \Bi{\Gr}$. If $(F,F')$ is a hedge for $\Px{\mathbf{x}'}{\mathbf{y}'}$ in $\Gr \gminus \Aug$, it is a hedge for the same query in $\Gr$.
\end{lemma}

\begin{proof}
The vertices of $F$ are joined by bi-directed paths in $\Gr \gminus \Aug$, whose bi-directed edges are a subset of those of $\Gr$, so they remain a C-component in $\Gr$; the same holds for $F'$. The forest condition and the root set $\Rs$ are determined by the directed edges of $F$, which the two graphs share. The conditions $F' \subseteq F$, $F \cap \Xs' \neq \emptyset$ and $F' \cap \Xs' = \emptyset$ involve vertex sets only. Finally, ancestry is a directed notion and $\Gr \gminus \Aug$ and $\Gr$ have the same directed edges, so $\An(\Ys')_{(\Gr \gminus \Aug)_{\overline{\Xs'}}} = \An(\Ys')_{\Gsub{\Xs'}}$ and the containment is preserved.
\end{proof}

\begin{proof}[Proof of Theorem~\ref{thm:destroy}]
Write $F^{\leftrightarrow}$ for the bi-directed edges of $F$.

\emph{Second clause first.} Since $\Gr \gminus \Aug$ has a subset of the edges of $\Gr$, $F$ and $F'$ are also subgraphs of $\Gr$, and their bi-directed edges still connect their vertices. Every condition in the definition of a hedge except the two C-component conditions depends only on vertices and directed edges, so those conditions hold for $(F \setminus \Aug, F' \setminus \Aug)$ exactly as for $(F,F')$. The C-component conditions hold if and only if $F^{\leftrightarrow} \setminus \Aug$ connects $V(F)$ and $F'^{\leftrightarrow} \setminus \Aug$ connects $V(F')$. Hence $(F \setminus \Aug, F' \setminus \Aug)$ fails to be a hedge exactly when deleting $\Aug$ disconnects $F^{\leftrightarrow}$ or $F'^{\leftrightarrow}$.

\emph{First clause.} If $\Aug \in \Adm(\Gr,\Qy)$, then $\Qy$ is identifiable in $\Gr \gminus \Aug$, and the hedge criterion gives $\mathcal{H}(\Gr \gminus \Aug, \Qy) = \emptyset$. In particular, no pair $(F \setminus \Aug, F' \setminus \Aug)$ is a hedge in $\Gr \gminus \Aug$, so $\Aug$ destroys every element of $\mathcal{H}(\Gr,\Qy)$. Conversely, suppose $\Aug$ destroys every element of $\mathcal{H}(\Gr,\Qy)$, and let $(F,F') \in \mathcal{H}(\Gr \gminus \Aug, \Qy)$. By Lemma~\ref{lem:persist}, $(F,F') \in \mathcal{H}(\Gr,\Qy)$. Since $F$ and $F'$ are subgraphs of $\Gr \gminus \Aug$, neither contains an edge of $\Aug$, so $(F \setminus \Aug, F' \setminus \Aug) = (F,F')$, which is a hedge in $\Gr \gminus \Aug$. This contradicts $\Aug$ destroying $(F,F')$. Hence $\mathcal{H}(\Gr \gminus \Aug, \Qy) = \emptyset$, and the hedge criterion gives identifiability.
\end{proof}

\begin{proof}[Proof of Corollary~\ref{cor:antichain}]
If $\Aug, \Aug' \in \Frontier(\Gr,\Qy)$ with $\Aug \subsetneq \Aug'$ then $\Aug'$ is not $\subseteq$-minimal in $\Adm(\Gr,\Qy)$, a contradiction; so no element of $\Frontier(\Gr,\Qy)$ properly contains another and the family is an antichain in the Boolean lattice on $\Bi{\Gr}$. Sperner's theorem bounds the size of an antichain on an $m$-element ground set by $\binom{m}{\lfloor m/2 \rfloor}$.
\end{proof}

\begin{proof}[Proof of Corollary~\ref{cor:membership}]
By Theorem~\ref{thm:destroy}, $\Aug \in \Adm(\Gr,\Qy)$ exactly when, for every $(F,F') \in \mathcal{H}(\Gr,\Qy)$, deleting $\Aug$ disconnects the bi-directed edges of $F$ or of $F'$. Each test is a connectivity query on a subgraph of $\Gr$ and runs in $O(|\Vs| + |\Bi{\Gr}|)$ by breadth-first search; there are two per hedge.
\end{proof}

\begin{proof}[Proof of Corollary~\ref{cor:hardness}]
Suppose the bi-directed structure of every $F$ in $\mathcal{H}(\Gr,\Qy)$ is a tree. Removing any edge of a tree disconnects it, so $\Aug$ disconnects $F^{\leftrightarrow}$ exactly when $\Aug$ meets its edge set. Since $F' \subseteq F$, we have $F'^{\leftrightarrow} \subseteq F^{\leftrightarrow}$, so if $\Aug$ disconnects $F'^{\leftrightarrow}$, it meets $F^{\leftrightarrow}$. By Theorem~\ref{thm:destroy}, $\Aug$ destroys $(F,F')$ exactly when it meets the bi-directed edge set of $F$, and $\Adm(\Gr,\Qy)$ is the family of sets hitting every set in $\{F^{\leftrightarrow} : (F,F') \in \mathcal{H}(\Gr,\Qy)\}$. Minimizing $|\Aug|$ over that family is an instance of minimum hitting set.
\end{proof}

\begin{proof}[Proof of Corollary~\ref{cor:ancestral}]
Let $(F,F') \in \mathcal{H}(\Gr,\Qy)$ be a hedge for $\Px{\mathbf{x}'}{\mathbf{y}'}$. By Definition~\ref{def:hedge}, the roots of $F$ are ancestors of $\Ys' \subseteq \Ys$. Every vertex of $F$ is an ancestor of a root: each vertex has at most one child in $F$, so following children in the acyclic $F$ ends at a root. Hence both endpoints of every edge in $F^{\leftrightarrow}$ and $F'^{\leftrightarrow}$ lie in $\An(\Ys)$.

Let $\Aug \in \Frontier(\Gr,\Qy)$, and suppose $e \in \Aug$ has an endpoint outside $\An(\Ys)$. Then $e$ lies in no $F^{\leftrightarrow}$ or $F'^{\leftrightarrow}$. Whether deleting $\Aug$ disconnects $F^{\leftrightarrow}$ depends only on $\Aug \cap F^{\leftrightarrow} = (\Aug \setminus \{e\}) \cap F^{\leftrightarrow}$, and likewise for $F'$. By Theorem~\ref{thm:destroy}, $\Aug \setminus \{e\}$ therefore destroys every hedge that $\Aug$ destroys, which is all of $\mathcal{H}(\Gr,\Qy)$, so $\Aug \setminus \{e\} \in \Adm(\Gr,\Qy)$. This contradicts the minimality of $\Aug$.
\end{proof}

\begin{proof}[Proof of Corollary~\ref{cor:lowerbound}]
By Theorem~\ref{thm:destroy}, $\Aug$ destroys $(F,F')$, so deleting $\Aug$ disconnects $F^{\leftrightarrow}$ or $F'^{\leftrightarrow}$. In the first case, $\Aug \cap F^{\leftrightarrow}$ is a set of edges of $F^{\leftrightarrow}$ whose deletion disconnects $V(F)$, so $|\Aug| \ge |\Aug \cap F^{\leftrightarrow}| \ge \lambda(F)$. The second case gives $|\Aug| \ge \lambda(F')$ in the same way.

For hedges $(F_1,F_1'), \dots, (F_r,F_r')$ with pairwise disjoint $F_i^{\leftrightarrow}$, apply the same argument to each. Since $F_i'^{\leftrightarrow} \subseteq F_i^{\leftrightarrow}$, it gives $|\Aug \cap F_i^{\leftrightarrow}| \ge \min\{\lambda(F_i), \lambda(F_i')\}$. Summing over the disjoint sets gives $|\Aug| \ge \sum_{i} \min\{\lambda(F_i), \lambda(F_i')\}$.
\end{proof}

\subsection{Grading}

Theorem~\ref{thm:grading} is stated relative to the family of SCMs the verifier draws from, so we fix that family first: discrete SCMs over $\Vs$ whose structural dependencies realize $\Gr$, with one latent variable per bi-directed edge and a fixed number of levels per variable, parametrized by their free conditional probability table entries. Each $\theta \in \Theta$ determines a model $\Scm_\theta \in \Mods{\Gr}$, so $\Theta$ indexes a subfamily of $\Mods{\Gr}$ and not all of it. Say that $\hat\est$ \emph{denotes} $\Qy$ on a class if $\hat\est(P_\Scm(\Vs)) = \Qy_\Scm$ for every model in that class.

\begin{proof}[Proof of Theorem~\ref{thm:grading}]
\emph{(i)} Suppose the procedure refutes, so $\hat\est(P_{\theta}(\Vs)) \neq \Qy_{\theta}$ for some drawn $\theta$. Since $\Scm_\theta \in \Mods{\Gr}$, this is a single model of $\Gr$ on which $\hat\est$ and $\Qy$ disagree, so $\hat\est$ does not denote $\Qy$ on $\Mods{\Gr}$.

\emph{(ii)} Write $D(\theta) = \hat\est(P_\theta(\Vs)) - \Qy_\theta$. In a discrete SCM with latents materialized, $P_\theta(\Vs)$ is obtained by summing products of conditional probability table entries over the latent configurations, so each of its entries is a polynomial in $\theta$; $\Qy_\theta$ is obtained the same way from the truncated factorization and is likewise polynomial. The parsed form of $\hat\est$ is built from entries of $P_\theta(\Vs)$ by addition, multiplication and division, so $\hat\est(P_\theta(\Vs))$ is a rational function of $\theta$ and $D$ is rational.

Suppose $\hat\est$ does not denote $\Qy$ on $\Theta$. Then $D$ is not identically zero, so writing $D = N/M$ with $N, M$ polynomials, we have $N \not\equiv 0$ on the region where $M \neq 0$. The zero set of a polynomial that is not identically zero has Lebesgue measure zero, so $\{\theta \in \Theta : D(\theta) = 0\}$ has measure zero, as does $\{\theta: M(\theta) = 0\}$, the set on which some conditioning event in $\hat\est$ has probability zero, and the expression is undefined. A draw absolutely continuous with respect to Lebesgue measure on $\Theta$ therefore assigns probability zero to both, and with $\tau = 0$ and exact arithmetic the draw refutes with probability one.
\end{proof}

\begin{remark}[What part (ii) does not say]
\label{rem:family}
Part (ii) is stated with respect to $\Theta$. An expression that denotes $\Qy$ on $\Theta$ while differing from it elsewhere in $\Mods{\Gr}$ is never refuted, and we do not establish that $\Theta$ separates $\Mods{\Gr}$: that would require showing the fixed latent cardinality is sufficient to witness every disagreement, which we have not done. Part (i) is unconditional.
\end{remark}

\begin{proof}[Proof of Corollary~\ref{cor:onedraw}]
By Theorem~\ref{thm:grading}(ii) the event that one draw fails to refute a $\hat\est$ not denoting $\Qy$ on $\Theta$ has probability zero. The event that all $k$ draws fail to refute is contained in the event that the first draw fails, so it also has probability zero, for every $k \ge 1$.
\end{proof}

\begin{proof}[Proof of Remark~\ref{rem:tau}]
For $\tau > 0$ put $E_\tau = \{\theta \in \Theta : |D(\theta)| \le \tau\}$. The family is nested decreasing as $\tau \downarrow 0$ with $\bigcap_{\tau > 0} E_\tau = \{D = 0\}$, which has measure zero by the proof of Theorem~\ref{thm:grading}. The draw is a probability measure, hence finite, so continuity from above gives $\Pr(E_\tau) \to 0$ as $\tau \to 0$.
\end{proof}

\subsection{Metrics}

\begin{proof}[Proof of Corollary~\ref{cor:acc-fc}]
A scored response is wrong exactly when it is a false claim or a miss. With $|\mathcal{I}| = |\mathcal{N}| = |\mathcal{S}|/2$, the number of wrong responses is $f|\mathcal{N}| + m|\mathcal{I}| = (f + m)|\mathcal{S}|/2$. Dividing by $|\mathcal{S}|$ gives $1 - a = (f + m)/2$. Since $m \ge 0$, $f \le 2(1 - a)$.

For $a \ge 1/2$, both bounds are attained. If all $(1 - a)|\mathcal{S}|$ wrong responses are misses, then $f = 0$ and $m = 2(1 - a) \le 1$. If all are false claims, then $f = 2(1 - a) \le 1$ and $m = 0$.
\end{proof}
\section{Structural Leakage in Semi-Synthetic Construction}
\label{app:leakage}

This appendix gives the feature-level diagnosis of the leak described in Sec.~\ref{sec:leakage}, the construction that reduces it, and the leakage that remains.

\subsection{Features and classifier}

For a bi-directed edge $\{u,v\}$ in an instance with query $P(y \mid \Do{x})$, write $\mathrm{pa}(\cdot)$ and $\mathrm{ch}(\cdot)$ for directed parents and children, $d(w) = |\mathrm{pa}(w)| + |\mathrm{ch}(w)|$ for the directed degree of $w$, and $b(w)$ for its number of bi-directed edges. Table~\ref{tab:leak-features} defines the thirteen features. All are permutation-invariant, since a bi-directed edge is an unordered pair; none depends on the variable names.

The classifier is a logistic regression on standardized features. Every reported AUC is held out: for each of the five source networks, the classifier is trained on the other four and scores the held-out network's edges, and the AUC is computed once over the pooled out-of-fold scores. A single-feature AUC instead uses the raw feature value as the score, with no model. A value below $0.5$ means the feature predicts real edges; both directions are leakage. We call a feature \emph{balanced} when its single-feature AUC lies in $[0.45, 0.55]$.

\begin{table}[h]
\centering\small
\caption{The thirteen structural features, and their single-feature AUC on our construction (927 edges). Values outside $[0.45, 0.55]$ are in bold.}
\label{tab:leak-features}
\begin{tabular}{lll}
\toprule
Feature & Definition & AUC \\
\midrule
directed degree sum      & $d(u) + d(v)$ & \textbf{0.574} \\
directed degree min      & $\min\{d(u), d(v)\}$ & 0.534 \\
directed degree max      & $\max\{d(u), d(v)\}$ & \textbf{0.580} \\
bi-directed degree sum   & $b(u) + b(v)$ & \textbf{0.442} \\
bi-directed degree min   & $\min\{b(u), b(v)\}$ & \textbf{0.410} \\
bi-directed degree max   & $\max\{b(u), b(v)\}$ & 0.480 \\
directed adjacency       & $1$ if a directed edge joins $u$ and $v$, else $0$ & 0.550 \\
directed distance        & shortest directed path, either direction; $-1$ if none & \textbf{0.637} \\
common parents           & $|\mathrm{pa}(u) \cap \mathrm{pa}(v)|$ & 0.472 \\
common children          & $|\mathrm{ch}(u) \cap \mathrm{ch}(v)|$ & \textbf{0.408} \\
touches query            & $1$ if $u$ or $v$ is $x$ or $y$, else $0$ & \textbf{0.562} \\
vertices                 & $|\Vs|$ & 0.505 \\
bi-directed edges        & $|\Bi{\Gr}|$ & 0.482 \\
\bottomrule
\end{tabular}
\end{table}

\subsection{The leak}

Consider a construction that hides a latent set in a published DAG, projects onto the remaining variables, and then injects spurious bi-directed edges uniformly at random among unconnected pairs. The label of every edge, real or spurious, is known exactly. On this construction (600 instances, 2{,}775 edges), the classifier separates real from spurious edges at $\mathrm{AUC} = \leakAUC{}$.

Three patterns produce the leak. A latent induces a bi-directed edge between every pair of its observed children, so real edges form cliques and concentrate bi-directed degree. The endpoints of a real edge are siblings under a common latent, so they are often joined by no directed path. Projecting a latent with an observed parent $P$ also adds a directed edge from $P$ to each observed child of the latent, so real pairs share observed parents. Uniform injection reproduces none of these patterns.

\subsection{The construction}

Two design choices reduce the leak. First, Definition~\ref{def:semi-family} restricts latent sets to root-closed ones, so projection adds no directed edges. Second, spurious edges are injected as the cliques of invented latents, each matched to real edges on a weighted set of ten structural features. On our construction (200 instances, 927 edges), the pooled leakage falls from \leakAUC{} to \shippedAUC{}, and the construction accepts \acceptRate{} of attempts. Table~\ref{tab:leak-shift} shows the change for the features that leaked most.

\begin{table}[h]
\centering\small
\caption{AUC under uniform injection and under our construction. Rows are single-feature AUCs; the last row is the held-out pooled classifier.}
\label{tab:leak-shift}
\begin{tabular}{lcc}
\toprule
& uniform injection & our construction \\
\midrule
common parents           & 0.266 & 0.472 \\
directed degree min      & 0.321 & 0.534 \\
bi-directed degree sum   & 0.299 & 0.442 \\
bi-directed degree min   & 0.258 & 0.410 \\
directed distance        & 0.585 & 0.637 \\
\midrule
pooled classifier        & \leakAUC{} & \shippedAUC{} \\
\bottomrule
\end{tabular}
\end{table}

\subsection{The leakage that remains}

Seven of the thirteen features remain unbalanced (Table~\ref{tab:leak-features}). The largest is directed distance, which the construction makes worse, from $0.585$ to $0.637$: $37\%$ of real edges join a pair with no directed path, against $14\%$ of spurious ones. The bi-directed-degree features stay below $0.5$, because real edges arrive in cliques and concentrate bi-directed degree, while matched spurious pairs are placed one at a time. Removing the remaining separation would require rejecting every instance in which some real pair has no exact structural match. That drops acceptance to $1.5\%$ and leaves graphs too small to be useful, so we use the construction above and report its leakage as a bar a system must exceed (Remark~\ref{rem:leak-bound}).
\section{Verification of the Labeling Procedure}
\label{app:verification}

The evaluation rests on two computations: the labels, computed by \ID{}, and the estimand grade, computed by the verifier. This appendix checks both, reports two defects the verifier found in a published library, and records which model produced each result.

\subsection{Canonical known-answer set}
\label{app:canonical}

We assembled \nCanonical{} instances whose verdicts are stated in published sources, transcribed independently of any package's test suite. Graph structures were read from source figures rendered at 600--1600 dpi, never from secondary descriptions. The sources are textbook cases encoded from first principles \citep{pearl1995causal, pearl2009causality}; Figures~1 and~2 of \citet{shpitser2008complete}, whose captions state the verdicts explicitly; Figures~1(a) and~1(b) of \citet{shpitser2006identification}, which differ by one directed edge and differ in verdict; and worked examples from \citet{tikka2017identifying}, transcribed from the printed \texttt{graph.formula} definitions.

One instance is flagged. The napkin graph's edge list could not be confirmed against a primary text, so it is recorded in our results as resting on a secondary source and is excluded from the headline count of \nCanonical{}; twenty-five instances are encoded in total. Both implementations return the correct verdict for it.

\subsection{Two implementations of \ID{}}
\label{app:implementations}

Labels are computed by two Python packages developed by different groups: \texttt{ananke} 0.5.0, which implements the one-line formulation of \ID{}, and \texttt{y0} 0.2.11, which implements \ID{} and its conditional extensions. The packages share no code. Our oracle calls both on every query and halts on any disagreement instead of choosing between them. Across $29{,}537$ random graphs, $19{,}537$ with 4 to 9 vertices and $10{,}000$ with 10 to 14, they disagree on none; across the ${\sim}\oracleCallTotal{}$ calls made while building and scoring the benchmark, they also disagree on none.

Both packages implement the same published algorithm, so their agreement cannot expose an error in that algorithm or a transcription mistake they share. We therefore also check both against answers taken from the primary papers, never from either package's tests, as described next.

We considered and rejected two other packages. The identification routine of \texttt{dowhy} 0.14 takes a directed graph with no way to mark a variable as unobserved, so it cannot represent an ADMG. The Python port \texttt{causal-effect} 0.0.2 crashes on about $30\%$ of instances. We did not test the reference R package \texttt{causal-effect}.

\subsection{The verifier}
\label{app:verifier}

Section~\ref{sec:verifier} describes the verifier. This subsection fixes the details needed to reproduce it and to check that Theorem~\ref{thm:grading} applies.

{\bf SCMs.} Each bi-directed edge $u \leftrightarrow v$ is materialized as one latent that is a parent of both $u$ and $v$. Every variable, observed or latent, is binary. For each variable and each configuration of its parents, the conditional distribution is drawn as $w_i = 0.12 + 0.76\,U_i$ for $i = 1, 2$, with $U_i \sim \mathrm{Uniform}(0,1)$ independent, and then normalized. Each latent's marginal is drawn the same way. The floor of $0.12$ keeps every probability away from zero, so no estimand divides by zero. This distribution is absolutely continuous, as Theorem~\ref{thm:grading}(ii) requires, and the guarantee holds relative to binary SCMs (Remark~\ref{rem:slice}).

{\bf Exact computation.} The joint distribution is computed by enumerating every configuration of the observed and latent variables, and each interventional distribution by removing the edges into the intervened variables. No quantity is estimated by sampling. Main-pool grading uses seeds 101 and 102, one per SCM; the known-answer checks and the case study use seeds 101 to 105.

{\bf Grammar.} A returned estimand must use four constructs: a probability term $P(\mathbf{a} \mid \mathbf{b})$ over observed variables, a product, a quotient, and a sum over named variables. Subtraction, expectations, numeric constants and other functions are rejected. The grammar covers every expression \ID{} returns, since \ID{} builds its output from exactly these constructs. Every variable must be an observed vertex of the instance, and an estimand may have at most 400 nodes and nesting depth 24. An estimand that fails to parse is counted as a parse failure and is not retried.

\subsection{Five defects in a published causal-inference library}
\label{app:librarybug}

We applied the verifier to the outputs of \texttt{y0}, version 0.2.11, a published causal inference library. On five graphs it returns estimands that are numerically incorrect, while its verdicts remain correct. We reported all five cases upstream.

On the napkin graph the library returns $P(Y \mid X)$ for $P(y \mid \Do{x})$, which is precisely the naive conditional the napkin construction exists to refute. That expression deviates from exact SCM ground truth by $7.5\times10^{-4}$. On Figure~7 of \citet{tikka2017identifying}, it returns a factorization that deviates from ground truth by \tikkaDev{} at the first of our five seeds, with a range of \tikkaDevRange{} across them; the formula published for that graph agrees with ground truth at ${\sim}10^{-16}$. Appendix~\ref{app:ex-estimand} works this case out in full.

The other three graphs are nine-vertex instances from our main pool. Over five seeds and both intervention values, the library's expressions deviate from ground truth by up to $2.3\times10^{-2}$, $8.1\times10^{-3}$ and $6.9\times10^{-2}$. On two of them, the factor for the outcome omits at least one of its parents: for $P(v_6 \mid \Do{v_4})$ on one graph, it conditions on $V_2$ and $V_5$ and omits the treatment $V_4$, a parent of $V_6$. On two of these instances, \modelFlash{} and \modelPro{} both return correct estimands where the library's is wrong.

All five outputs sum to one, so no discrepancy is a normalization artifact. We confirmed each by hand-coding the returned expression independently of our parser, and the library's own shipped example of the napkin graph reproduces the same output. All five magnitudes come from seeded draws and reproduce from the artifact. The artifact's results file also lists one mismatch for an instance that was never evaluated; it is a bookkeeping entry, not a defect. Our pipeline uses only the library's identifiability verdicts, which are correct in all five cases, so no label is affected. We report the defects because they illustrate the argument before Definition~\ref{def:verified}: an expression that looks like a plausible estimand, produced by trusted software, can be wrong in a way that only numerical evaluation exposes.

\subsection{Model identities}
\label{app:models}

\begin{center}\small
\begin{tabular}{llc}
\toprule
Identifier & Pin quality & Run dates \\
\midrule
\texttt{gemini-3.7-flash} & floating alias, requested only & 2026-08-17/18 \\
\texttt{gemini-3.1-pro-preview} & floating preview alias, requested only & 2026-08-17/18 \\
\texttt{gpt-5.5-2026-04-23} & dated snapshot, echo-verified & 2026-08-20 \\
\bottomrule
\end{tabular}
\end{center}

The three differ in how firmly they are pinned. For \texttt{gpt-5.5-2026-04-23} the API returns the served model identifier on every call, and we log the identifier per call. Coverage is 1{,}189 of 1{,}189, so the pin is verifiable from the logs and not merely requested. The Gemini client captured no served-model field, so those results rest on the requested alias alone, and one of the two is an explicitly preview endpoint. Neither Gemini result can be re-run against identical weights.

{\bf Reasoning.} No run set a thinking budget or reasoning effort; every call used the vendor's default, and Gemini calls used temperature $0$. Reasoning ran on every call: no logged reply has zero reasoning tokens. Table~\ref{tab:reasoning} gives the mean reasoning tokens per reply. The vendor chooses the reasoning depth, so it differs between systems, by about $3\times$ between \modelFlash{} and \modelPro{}, and is not held fixed in any comparison.

\begin{table}[h]
\centering\small
\caption{Mean reasoning tokens per reply.}
\label{tab:reasoning}
\begin{tabular}{lcc}
\toprule
& main pool & scale grid \\
\midrule
\modelFlash{} & 2{,}421 & 4{,}114 \\
\modelPro{}   & 8{,}070 & 17{,}144 \\
\modelGPT{}   & 2{,}994 & 5{,}733 \\
\bottomrule
\end{tabular}
\end{table}

The asymmetry deserves naming: the model with the strongest reproducibility guarantee is the one whose results differ most from the other two (Sec.~\ref{sec:res-decouple}). A reader inclined to discount the cross-model spread should note that discounting it requires trusting the two endpoints we can \emph{less} well account for. Open-weights models, pinned by full revision hash:

\begin{center}\scriptsize
\begin{tabular}{ll}
\toprule
Model & Revision \\
\midrule
\texttt{Qwen/Qwen2.5-0.5B-Instruct} & \texttt{7ae557604adf67be50417f59c2c2f167def9a775} \\
\texttt{Qwen/Qwen2.5-1.5B-Instruct} & \texttt{989aa7980e4cf806f80c7fef2b1adb7bc71aa306} \\
\texttt{Qwen/Qwen2.5-3B-Instruct}   & \texttt{aa8e72537993ba99e69dfaafa59ed015b17504d1} \\
\texttt{sentence-transformers/all-MiniLM-L6-v2} & \texttt{1110a243fdf4706b3f48f1d95db1a4f5529b4d41} \\
\bottomrule
\end{tabular}
\end{center}
\section{Evaluation Protocol and Additional Results}
\label{app:truncation}

Three things determine what a reported figure means: the budget a model was given, the form in which it saw the graph, and the rule by which its reply was turned into a verdict. None is neutral. The budget and the graph format differ across runs, and the parser resolves different models by different routes. This appendix records them. It also gives in full the results the body states in summary.

\subsection{Published networks}
\label{app:networks}
 
Table~\ref{tab:networks} lists the five networks from which the published and semi-synthetic instances are built. All five come from the bnlearn repository~\citep{scutari2010learning}, loaded at a pinned library version. The networks carry no version string, so the loader checks each network's node and edge counts against the repository's published figures and refuses to proceed if they differ.
 
\begin{table}[h]
\centering\footnotesize
\caption{The five networks behind the published family (Definition~\ref{def:semi-family}).}
\label{tab:networks}
\begin{tabular}{lllrr}
\toprule
Network & Domain & Source & Nodes & Edges \\
\midrule
\textsc{asia}      & lung-disease diagnosis   & \citet{lauritzen1988local}        &  8 &   8 \\
\textsc{child}     & congenital heart disease & \citet{spiegelhalter1992learning} & 20 &  25 \\
\textsc{insurance} & car-insurance risk       & \citet{binder1997adaptive}        & 27 &  52 \\
\textsc{alarm}     & patient monitoring       & \citet{beinlich1989alarm}         & 37 &  46 \\
\textsc{hepar2}    & liver-disorder diagnosis & \citet{onisko2003probabilistic}   & 70 & 123 \\
\bottomrule
\end{tabular}
\end{table}

\subsection{Output budgets}
\label{app:proto-budget}

All three models draw internal reasoning from the same budget as visible output, so a budget sized for the answer alone produces truncation or, in one case, a request error with no content at all. An early run at a smaller budget truncated a majority of \modelPro{} responses and inverted one of our own reported figures, because the instances dropped were disproportionately those the model reasoned longest about. Budgets were raised and differ by run: $16{,}384$ tokens for the main-pool verdict task; $32{,}768$ for the scale grid, the second-vendor run, the repeat-sampling study and the justification variant; and per row as given in Table~\ref{tab:repair} for repair. The oracle condition allows fifteen calls to \ID{} and was run on 150 of the 300 main-pool non-identifiable instances, of which 144 are distinct.

\subsection{Graph serialization}
\label{app:proto-serial}

Two renderings of the graph exist in the code base. The main-pool verdict runner emits one edge per line with explanatory parentheticals; every other runner emits a compact form, one line per edge type. The prompt template around the graph block, and the response parser, are identical throughout.

The consequence is confined to one comparison. On the scale grid all three models received prompts identical in bytes: we regenerated the expected prompt for all 600 instances and matched its hash in each model's log, finding 600 of 600 in each and no other prompt hash in any of them. On the main pool, \modelGPT{} received the compact form and the two Gemini endpoints the multi-line form, so the cross-vendor comparison in Table~\ref{tab:verdict} holds the template and the parser fixed while the graph block differs in layout. The designed repeat-sampling study used the compact form throughout.

\subsection{How a verdict is read}
\label{app:proto-routes}

The parser resolves a response to a verdict by three routes, tried in order. The first is an explicit verdict line, which in every case we observed has the strict format \texttt{VERDICT: IDENTIFIABLE} or \texttt{VERDICT: NOT IDENTIFIABLE}. The second is a fenced JSON block parsing under the estimand grammar, read as a claim of identifiability. The third is a scan of the whole response for the word \emph{identifiable}. The second route can produce a claim and never a refusal, so it is not symmetric across labels; the third is a fallback we did not document in advance and which fired only on truncated text.

Table~\ref{tab:routes} gives the decomposition for every model and pool. The reading is straightforward: \modelFlash{} and \modelGPT{} are resolved entirely by explicit verdict lines, in both pools and on both labels, and neither fallback fired for either model once. \modelPro{} is the only model with mixed routes. This matters most for Sec.~\ref{sec:res-decouple}, whose headline comparison is \modelFlash{} against \modelGPT{} on the scale grid. All $51$ and all $3$ of those false claims are explicit verdict lines, as are all $249$ and $297$ correct refusals behind the denominators. That comparison is therefore like-for-like in route as well as in instance and prompt.

\begin{table}[h]
\centering\small
\caption{How each verdict was read, by model, pool and label. Counts are of
responses; the shipped metrics reproduce exactly from them.}
\label{tab:routes}
\begin{tabular}{llcccc}
\toprule
& & total & verdict line & estimand block & whole-text scan \\
\midrule
\multicolumn{6}{l}{\emph{false claims}} \\
\textsc{flash}   & main  & 9   & 9   & 0 & 0 \\
\textsc{flash}   & scale & 51  & 51  & 0 & 0 \\
\textsc{pro}     & main  & 8   & 4   & 2 & 2 \\
\textsc{pro}     & scale & 22  & 16  & 6 & 0 \\
\textsc{gpt-5.5} & main  & 0   & 0   & 0 & 0 \\
\textsc{gpt-5.5} & scale & 3   & 3   & 0 & 0 \\
\midrule
\multicolumn{6}{l}{\emph{correct refusals}} \\
\textsc{flash}   & main  & 291 & 291 & 0 & 0 \\
\textsc{flash}   & scale & 249 & 249 & 0 & 0 \\
\textsc{pro}     & main  & 283 & 281 & 0 & 2 \\
\textsc{pro}     & scale & 276 & 276 & 0 & 0 \\
\textsc{gpt-5.5} & main  & 300 & 300 & 0 & 0 \\
\textsc{gpt-5.5} & scale & 297 & 297 & 0 & 0 \\
\bottomrule
\end{tabular}
\end{table}

Two further runs were checked the same way. In the justification variant no response from any model contains a JSON block, so neither fallback could fire, and every verdict is an explicit line; \modelFlash{}'s rise from nine false claims to fourteen is therefore five additional stated claims and not an artifact of the longer replies. In the repeat-sampling study all 600 \modelFlash{} samples are verdict lines, and all nine of its non-unanimous instances are changes of stated verdict. \modelPro{} has two samples resolved by the estimand route, one of which falls on a non-unanimous instance: on an identifiable instance two samples state \texttt{NOT IDENTIFIABLE} and the third returns a bare estimand read as a claim. That is a change of route and not of stated verdict, and Appendix~\ref{app:proto-repro} reports \modelPro{}'s figure both ways.

\subsection{Truncation}
\label{app:proto-trunc}

\modelFlash{} and \modelGPT{} have no truncated responses in either pool: every scale-grid row carries a populated stop reason of \texttt{STOP}, and \modelFlash{}'s longest response is $15{,}089$ tokens of $16{,}384$ on the main pool and $13{,}973$ of $32{,}768$ on the scale grid.

\modelPro{} has thirteen main-pool responses cut at the ceiling of $16{,}380$ output-plus-reasoning tokens. Seven yielded no verdict and are among its eleven unscored instances, alongside one formatting failure and three transport errors. Five were resolved by the whole-text scan, two of which are counted as false claims in the shipped metrics and three as correct refusals. One had completed an estimand block before the cut. On the scale grid \modelPro{} has three truncated responses, at $n{=}40$ and $n{=}50$. The $0.0000$ truncation rate in our shipped scale metrics is a caching artifact: 599 of its 600 rows were served from a log that predates the stop-reason field, and a null value was counted as not truncated.

Table~\ref{tab:trunc-effect} gives the effect of excluding the truncated responses. No comparison in the paper changes, and we report the shipped figures throughout so that every number rests on one treatment.

\begin{table}[h]
\centering\small
\caption{\modelPro{} under the shipped treatment and with truncated responses
excluded. \modelFlash{} and \modelGPT{} are unaffected, having none.}
\label{tab:trunc-effect}
\begin{tabular}{lcc}
\toprule
\modelPro{} & as shipped & truncated excluded \\
\midrule
main-pool accuracy        & 0.9779 (589 scored) & 0.9811 (583 scored) \\
main-pool false claims    & 8/291 = 0.0275      & 6/287 = 0.0209 \\
scale-grid accuracy       & 0.9415              & 0.9430 \\
scale-grid false claims   & 22/298 = 0.0738     & 21/297 = 0.0707 \\
false claims at $n{=}50$  & 6/49 = 0.1224       & 5/48 = 0.1042 \\
\bottomrule
\end{tabular}
\end{table}

\subsection{Small open-weights models in full}
\label{app:proto-small}

\begin{table}[h]
\centering\small
\caption{Small open-weights models on the edge-classification task, against name-blind and lexical baselines. Values come from the uniform-injection construction of Appendix~\ref{app:leakage}; the structural baseline is $0.821$ under the original features and \leakAUC{} under the permutation-invariant features of Table~\ref{tab:leak-features}.Values come from the superseded construction of Appendix~\ref{app:leakage}. \textbf{Comparisons are valid within a column only}: the published and scrambled columns are scored on the same 2{,}775 edge instances, the novel column on a separate set of 731, and none of these are comparable to the frontier-model results elsewhere in this paper. The 1.5B model discriminates in the wrong direction; scrambling variable names returns it to chance, so the inversion is name-driven.}
\label{tab:smallmodels}
\begin{tabular}{lccc}
\toprule
Scorer & published & scrambled & novel domains \\
\midrule
Random                & 0.4953 & 0.4953 & 0.5342 \\
Embedding similarity  & 0.5047 & 0.4115 & 0.5590 \\
Structural (no names) & 0.8210 & 0.8210 & 0.7487 \\
\midrule
0.5B & 0.4405 & 0.4770 & 0.4529 \\
1.5B & \textbf{0.3586} & 0.4929 & 0.5363 \\
3B   & 0.5181 & 0.4427 & 0.5100 \\
\bottomrule
\end{tabular}
\end{table}

On edge classification, the small models show no useful signal (Table~\ref{tab:smallmodels}). Across \smallModelRange{} models, performance sits at or below chance in every condition, with no trend across the three sizes. The one strongly non-chance value is an inversion: the 1.5B model reaches AUC \smallWorstAUC{} on named instances, as far below chance as $0.641$ would be above it, and its scrambled-name control returns to \smallScrambledAUC{}. The inversion is name-driven and not noise, but it appears in a single cell and we report it as an isolated observation and not a mechanism.

\subsection{Scale grid, numerically}
\label{app:proto-scale}
\begin{table}[h]
\centering\small
\caption{Verdict accuracy against graph size.}
\label{tab:scale}
\begin{tabular}{lccccccc}
\toprule
& \multicolumn{2}{c}{\textsc{flash}} & \multicolumn{2}{c}{\textsc{pro}} & \multicolumn{2}{c}{\textsc{gpt-5.5}} \\
\cmidrule(lr){2-3}\cmidrule(lr){4-5}\cmidrule(lr){6-7}
$n$ & Acc. & False-claim & Acc. & False-claim & Acc. & False-claim \\
\midrule
10 & 0.940 & 0.12 & 0.960 & 0.02 & \textbf{1.000} & 0.00 \\
15 & 0.960 & 0.06 & 0.960 & 0.04 & \textbf{1.000} & 0.00 \\
20 & 0.890 & 0.20 & 0.960 & 0.06 & 0.990 & 0.02 \\
30 & 0.840 & 0.28 & 0.940 & 0.10 & \textbf{1.000} & 0.00 \\
40 & 0.920 & 0.16 & 0.919 & 0.10 & \textbf{1.000} & 0.00 \\
50 & 0.840 & 0.20 & 0.909 & 0.12 & 0.980 & 0.04 \\
\midrule
overall & 0.8983 & 0.1700 & 0.9415 & 0.0738 & \textbf{0.9950} & \textbf{0.0100} \\
\bottomrule
\end{tabular}
\end{table}

\vspace{-2mm}
\subsection{Accuracy and soundness under scale, in full}
\label{app:res-decouple}
\vspace{-2mm}

Across graph sizes from \nVertexRange{} vertices, the Gemini models lose accuracy slowly while their false-claim rates rise, and \modelGPT{} does neither (Table~\ref{tab:scale}). For \modelPro{}, accuracy falls from \decoupleAccSmall{} to \accDecoupleAcc{}, five points, while the false-claim rate rises from \decoupleFCSmall{} to \accDecoupleFC{}, a six-fold increase. \modelFlash{} shows the same divergence more steeply. \modelGPT{} holds at \gptScaleAcc{} accuracy with a false-claim rate of \gptScaleFC{}; its three errors across six sizes form no trend. Excluding \modelPro{}'s three truncated responses changes no comparison (Table~\ref{tab:trunc-effect}).

With fifty non-identifiable instances per size, the per-size intervals overlap for adjacent sizes, so the shape of the trend within the grid is suggestive, not established. The pooled comparison is resolved: \modelFlash{}'s false-claim rate is $0.170$ $[0.127, 0.217]$ on the scale grid, against $0.030$ on the main pool.

The relevant subgraph grows far more slowly than the graph. Identifiability depends only on the subgraph over $\An(y)$, since the first step of \ID{} discards every other vertex~\citep{shpitser2006identification}, and by Corollary~\ref{cor:ancestral} no minimal repair uses an edge outside it. From $|\Vs| = 10$ to $50$, mean $|\An(y)|$ rises only from $6.9$ to $15.1$, and the number of bi-directed edges inside $\An(y)$ stays nearly flat, from $4.6$ to $5.7$. Enlarging the graph therefore adds mostly vertices that are irrelevant to the query, so the degradation is consistent with sensitivity to irrelevant structure more than with a harder identification problem. The spread is wide at every size ($|\An(y)|$ ranges from 4 to 34 at $|\Vs| = 50$), so size is a coarse proxy for difficulty.

\subsection{Repair and enumeration in full}
\label{app:proto-repair}

Table~\ref{tab:repair} reports repair and enumeration on the non-identifiable instances of the main pool, and Table~\ref{tab:repairscale} reports \modelFlash{} against graph size. The trivial repair $\Aug = \Bi{\Gr}$ is always valid but minimal on only \trivialMin{} of instances, so validity alone is a weak target; the systems' valid repairs are minimal on \repairMin{} of instances. With \ID{} as a tool, every repair \modelFlash{} returns is valid and minimal, at every size. Exhaustive enumeration of $\Frontier$ becomes infeasible beyond 20 vertices, so frontier coverage and mean $|\Frontier|$ at larger sizes describe only the instances where enumeration finished.

\begin{table}[h]
\centering\small
\caption{Repair and enumeration on the main pool.}
\label{tab:repair}
\begin{tabular}{llcccccc}
\toprule
Model & Condition & Budget & Parse err. & Validity & Min.\ $\mid$ valid & Exact $\Frontier$ & Recall \\
\midrule
\textsc{flash} & one-shot      & 8{,}192  & 0.007 & 0.9295 & 0.9675 & 0.8993 & --- \\
\textsc{flash} & asked-for-all & 32{,}768 & 0.003 & 0.9264 & 0.9928 & 0.9197 & 0.9066 \\
\textsc{flash} & oracle        & 8{,}192  & 0.053 & 1.0000 & 1.0000 & 1.0000 & --- \\
\textsc{pro}   & one-shot      & 32{,}768 & 0.010 & 0.9697 & 0.9826 & 0.9529 & --- \\
\textsc{pro}   & asked-for-all & 32{,}768 & 0.030 & 0.9485 & 0.9746 & 0.9244 & 0.8704 \\
\textsc{pro}   & oracle        & 8{,}192  & 0.380 & \multicolumn{4}{c}{excluded: parse failure on 38\% of instances} \\
\midrule
trivial ($\Aug = \Bi{\Gr}$) & & --- & --- & 1.0000 & 0.0833 & --- & --- \\
\bottomrule
\end{tabular}
\end{table}

\begin{table}[h]
\centering\small
\caption{Repair against graph size (\textsc{flash}).}
\label{tab:repairscale}
\begin{tabular}{lcccccc}
\toprule
$n$ & Valid & Min.\ $\mid$ valid & Oracle valid & Oracle min. & $\Frontier$ coverage & mean $|\Frontier|$ \\
\midrule
10 & 0.880 & 0.9545 & 1.000 & 1.000 & 1.00 & 3.28 \\
15 & 0.860 & 0.9535 & 1.000 & 1.000 & 1.00 & 3.73 \\
20 & 0.860 & 0.8837 & 1.000 & 1.000 & 0.88 & 2.30 \\
30 & 0.880 & 0.7500 & 1.000 & 1.000 & 0.10 & 1.40 \\
40 & 0.920 & 0.8696 & 1.000 & 1.000 & 0.02 & 1.00 \\
50 & 0.980 & 0.8163 & 1.000 & 1.000 & 0.00 & --- \\
\bottomrule
\end{tabular}
\end{table}

\subsection{Reproducibility at temperature zero}
\label{app:proto-repro}

Temperature zero did not produce deterministic output in either Gemini model. Drawing three fresh responses for each of \nRepeatInst{} instances stratified evenly by label, \modelFlash{} returns a non-unanimous verdict on \flipFlash{} of instances (95\% CI \flipFlashCI{}) and \modelPro{} on \flipPro{} (CI \flipProCI{}); pairwise flip rates are \flipPairFlash{} and \flipPairPro{}. Two properties of the instability matter more than its size. One of \modelPro{}'s five non-unanimous instances is a change of parser route and not of stated verdict; setting it aside gives $4/200$ with a pairwise rate of $8/600$ (Appendix~\ref{app:truncation}). \modelFlash{}'s nine are all changes of stated verdict.

It is not symmetric across labels. \modelFlash{} is perfectly unanimous on all one hundred identifiable instances and unstable on nine of the hundred non-identifiable ones, so its entire instability sits where a flip converts a correct refusal into a false claim. \modelPro{}'s flips are spread across both labels, three and two. And it tracks difficulty as accuracy does: \modelFlash{}'s instability concentrates in the middle density band ($1.1\%$ below six bi-directed edges, $8.8\%$ between six and twelve, $5.5\%$ above), the same band in which its verdict errors concentrate.

Majority voting over the three samples does not repair this. Accuracy moves by $+0.8$ points for \modelFlash{} and $-0.2$ for \modelPro{}: neither is meaningful at this sample size, and one has the wrong sign, so three times the inference budget recovers nothing. The consequence is methodological and general. A single-run accuracy figure from these endpoints is a sample and not a constant, and a gap of one or two points between two systems, or between a system and itself, lies within the noise of re-running. We report single-sample figures throughout and mark them as such, and we suggest that repeat variance should be standard for benchmarks evaluated against commercial endpoints.
\section{Worked Examples}
\label{app:examples}

The paper claims that two things are computable: whether a query is identifiable, and whether a returned expression is correct. This appendix works one instance through each, then applies both computations to a graph an author published. The first instance is drawn from the pool; the second is a known-answer case from the literature.

\subsection{A certified label: a non-identifiable instance and its frontier}
\label{app:ex-label}

Figure~\ref{fig:ex-child} shows an instance of the published family drawn from the \textsc{child} network, which models congenital heart disease in newborns. Five observed variables remain after projection: cardiac mixing (\textsf{CM}), duct flow (\textsf{DF}), hypoxia distribution (\textsf{HD}), hypoxia in O\textsubscript{2} (\textsf{HO}) and lower-body O\textsubscript{2} (\textsf{LB}). The query is $P(\textsf{LB} \mid \Do{\textsf{CM}})$.

\begin{figure}[h]
\centering
\scalebox{0.92}{
\begin{tikzpicture}[
    every node/.style={font=\small},
    v/.style={circle, draw, thick, minimum size=0.78cm, inner sep=0pt,
              font=\scriptsize},
    d/.style={-{Stealth}, thick},
    b/.style={{Stealth}-{Stealth}, dashed, thick, red!65},
    box/.style={rectangle, draw, rounded corners, align=left,
                font=\scriptsize, inner sep=4pt},
    good/.style={box, fill=green!12, draw=green!60},
    bad/.style={box, fill=red!10, draw=red!60}
]
\node[v] (CM) at (0.5,2.5)  {CM};
\node[v] (DF) at (0.5,0.5)  {DF};
\node[v] (HD) at (2.6,2.5)  {HD};
\node[v] (HO) at (2.6,0.5)  {HO};
\node[v] (LB) at (4.6,1.5)  {LB};

\draw[d] (CM) -- (HD);
\draw[d] (CM) -- (HO);
\draw[d] (DF) -- (HD);
\draw[d] (HD) -- (LB);
\draw[d] (HO) -- (LB);

\draw[b] (CM) to[bend right=50] (DF);
\draw[b] (CM) to[bend left=30]  (HO);
\draw[b] (DF) -- (HO);

\node[font=\scriptsize, align=center] at (2.5,-0.55)
  {$\Qy = P(\textsf{LB} \mid \Do{\textsf{CM}})$ \quad
   \textsc{id}: \textbf{not identifiable} \quad (both implementations)};

\node[good, minimum width=5.0cm] at (9.0,2.7)
  {$F_1 = \{\textsf{CM}\!\leftrightarrow\!\textsf{DF},\ \textsf{CM}\!\leftrightarrow\!\textsf{HO}\}$\\
   the treatment is unconfounded};
\node[good, minimum width=5.0cm] at (9.0,1.45)
  {$F_2 = \{\textsf{CM}\!\leftrightarrow\!\textsf{HO},\ \textsf{DF}\!\leftrightarrow\!\textsf{HO}\}$\\
   the mediator is unconfounded};
\node[bad, minimum width=5.0cm] at (9.0,0.15)
  {a model proposes $\{\textsf{CM}\!\leftrightarrow\!\textsf{HO}\}$\\
   \textbf{invalid: one call to \ID{}}};
\end{tikzpicture}
}
\caption{A certified instance, its frontier, and an invalid repair. Dashed edges are unobserved common causes.}
\label{fig:ex-child}
\end{figure}
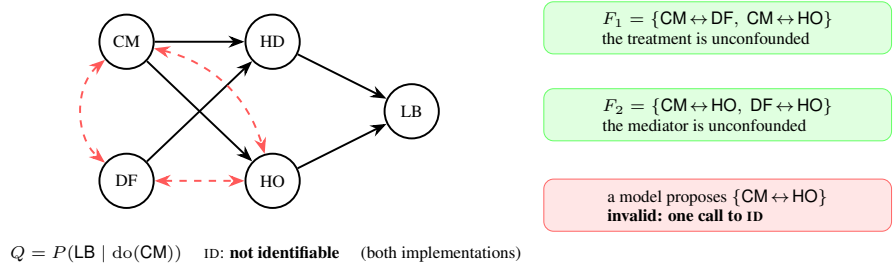

\paragraph{The label.} Both implementations return $\FAIL$. The obstruction is a hedge (Definition~\ref{def:hedge}): $F$ on $\{\textsf{CM},\textsf{DF},\textsf{HO}\}$ and $F'$ on $\{\textsf{DF},\textsf{HO}\}$ are C-forests with the same roots $\{\textsf{DF},\textsf{HO}\}$, both ancestors of $\textsf{LB}$; $F$ contains the treatment and $F'$ does not. One of the three frontier models named exactly this pair when asked for a justification (Sec.~\ref{sec:res-warrant}); all three returned the correct verdict.

\paragraph{The frontier.} Exhaustive search over the $2^3$ augmentations, seven calls to \ID{}, returns $\Frontier(\Gr,\Qy) = \{F_1, F_2\}$. No single deletion suffices: each of the three singletons leaves the query non-identifiable. The third pair, $\{\textsf{CM}\!\leftrightarrow\!\textsf{DF}, \textsf{DF}\!\leftrightarrow\!\textsf{HO}\}$, also fails.

The two elements have equal cardinality and incompatible meanings. $F_1$ asserts that nothing unobserved drives both cardiac mixing and either of its neighbors, so the treatment is unconfounded. $F_2$ asserts that nothing unobserved drives hypoxia in O\textsubscript{2} together with cardiac mixing or duct flow, which makes duct flow a valid adjustment. A procedure ranking repairs by size cannot choose between them, and nothing in the graph can: the choice is a clinical judgment about which confounding is credible. The same pattern recurs in the published graph of Appendix~\ref{app:ex-published}.

\paragraph{An invalid repair.} Asked for a minimal restoring set, one model proposed the single edge $\{\textsf{CM}\!\leftrightarrow\!\textsf{HO}\}$. The model argued that this edge carries the only confounding path, since a collider blocks the remaining path through duct flow. One call to \ID{} shows the proposal invalid (Definition~\ref{def:restoring}), and Theorem~\ref{thm:destroy} shows why: after deleting $\textsf{CM}\!\leftrightarrow\!\textsf{HO}$, the edges $\textsf{CM}\!\leftrightarrow\!\textsf{DF}$ and $\textsf{DF}\!\leftrightarrow\!\textsf{HO}$ still connect $F$, and $F'$ is untouched, so the hedge survives. The set is not in $\Adm(\Gr,\Qy)$, let alone minimal. Two properties of this failure are worth recording. It is a repair error and not a verdict error: the same model called the instance non-identifiable correctly. The graph also occurs twice in the pool, one of eight duplicated keys, and the same wrong repair was returned both times under different reasoning. This is a small instance of the instability reported in Appendix~\ref{app:proto-repro}.

\subsection{A false claim by the strongest model}
\label{app:ex-falseclaim}

\modelGPT{} makes three false claims on the scale grid (Table~\ref{tab:scale}). None is truncated or read by a fallback route: each reply opens with \texttt{VERDICT: IDENTIFIABLE} and returns a parseable estimand. We work through the one with the smallest hedge.

The instance has 50 vertices, and the query is $P(v_{41} \mid \Do{v_{21}})$. \ID{} returns $\FAIL$ with a hedge on four vertices: $F$ on $\{V_7, V_{13}, V_{21}, V_{28}\}$ and $F'$ on $\{V_7, V_{13}, V_{28}\}$, with bi-directed edges
\[
F^{\leftrightarrow} = \{V_7\!\leftrightarrow\!V_{28},\ V_{13}\!\leftrightarrow\!V_{28},\ V_{21}\!\leftrightarrow\!V_{28}\},
\qquad
F'^{\leftrightarrow} = F^{\leftrightarrow} \setminus \{V_{21}\!\leftrightarrow\!V_{28}\}.
\]
$F$ contains the treatment $V_{21}$ and $F'$ does not; a single bi-directed edge, $V_{21}\!\leftrightarrow\!V_{28}$, joins the treatment to the rest of the hedge. The hedge lies upstream of the outcome: $V_{41}$ is not in $F$, and the roots of the hedge are its ancestors. The reply states the verdict and returns a chain-rule expansion that conditions each variable on all its predecessors.

Both edge sets are trees, so by Theorem~\ref{thm:destroy} deleting any one of the three edges destroys this hedge (Corollary~\ref{cor:hardness}). The other two false claims have the same structure: the hedge contains the treatment and is joined to it by one or two bi-directed edges. The replies return a truncated factorization in one case and an adjustment for all 32 other ancestors of the outcome in the other. Accuracy counts each of these three answers as one error, the same as a miss, although each asserts an estimand where none exists.

\subsection{A verified grade: an estimand that is proper and wrong}
\label{app:ex-estimand}

The second example shows the grading side. Figure~7 of \citet{tikka2017identifying} carries five observed variables with $\textsf{X} \to \textsf{Z}_1 \to \textsf{Y}$, $\textsf{Z}_2 \to \{\textsf{X}, \textsf{Z}_1, \textsf{Z}_3\}$ and $\textsf{Z}_3 \to \textsf{Y}$, with bidirected edges $\textsf{X} \leftrightarrow \textsf{Y}$, $\textsf{X} \leftrightarrow \textsf{Z}_2$, $\textsf{X} \leftrightarrow \textsf{Z}_3$ and $\textsf{Y} \leftrightarrow \textsf{Z}_2$. The query is $P(z_1, z_2, z_3, y \mid \Do{x})$, which \citet{tian2002general} proved identifiable; both our implementations agree.

Two expressions are in play. The published estimand is
\[
  P_x(z_1,z_2,z_3,y) \;=\;
  P(z_1 \mid x, z_2)\, \textstyle\sum_{x'} P(y, z_3 \mid x', z_1, z_2)\, P(x', z_2),
\]
and \texttt{y0} 0.2.11 (Appendix~\ref{app:librarybug}) returns instead
\[
  \widehat{\est} \;=\;
  P(z_1 \mid x, z_2)\, P(z_3 \mid z_2)\, P(y \mid z_2, z_3)\, P(z_2).
\]
Both are proper distributions: each sums to one over the sixteen configurations, to within $10^{-12}$. Nothing about their surface form marks either as wrong, and a grader comparing strings would have to decide which is the reference.

\paragraph{The evaluation.} We draw a binary discrete SCM realizing the graph, with one binary latent per bi-directed edge, and enumerate the joint exactly. Table~\ref{tab:ex-estimand} gives every cell at $\Do{x{=}1}$ for the first of our five seeds. The published estimand agrees with ground truth to $5.6 \times 10^{-17}$ in every cell. The library's expression deviates by up to \tikkaDev{}, and across the five seeds the maximum deviation ranges over \tikkaDevRange{}.

\begin{table}[h]
\centering\small
\caption{$P(z_1,z_2,z_3,y \mid \Do{x{=}1})$ under one seeded SCM. Ground truth by exact enumeration; the published estimand matches it to $5.6\times10^{-17}$ throughout and is omitted. Deviations cancel within each
$(z_1,z_2,z_3)$ block.}
\label{tab:ex-estimand}
\begin{tabular}{cccc rrr}
\toprule
$z_1$ & $z_2$ & $z_3$ & $y$ & truth & library & difference \\
\midrule
0 & 0 & 0 & 0 & 0.068800 & 0.066377 & $-0.002423$ \\
0 & 0 & 0 & 1 & 0.087453 & 0.089876 & $+0.002423$ \\
0 & 0 & 1 & 0 & 0.073608 & 0.066449 & $-0.007159$ \\
0 & 0 & 1 & 1 & 0.083090 & 0.090249 & $+0.007159$ \\
0 & 1 & 0 & 0 & 0.028730 & 0.028270 & $-0.000460$ \\
0 & 1 & 0 & 1 & 0.039616 & 0.040076 & $+0.000460$ \\
0 & 1 & 1 & 0 & 0.050490 & 0.045796 & $-0.004694$ \\
0 & 1 & 1 & 1 & 0.059443 & 0.064138 & $+0.004694$ \\
1 & 0 & 0 & 0 & 0.063194 & 0.065219 & $+0.002025$ \\
1 & 0 & 0 & 1 & 0.090332 & 0.088308 & $-0.002025$ \\
1 & 0 & 1 & 0 & 0.056682 & 0.065289 & $+0.008608$ \\
1 & 0 & 1 & 1 & 0.097282 & 0.088674 & $-0.008608$ \\
1 & 1 & 0 & 0 & 0.031453 & 0.031917 & $+0.000464$ \\
1 & 1 & 0 & 1 & 0.045710 & 0.045246 & $-0.000464$ \\
1 & 1 & 1 & 0 & 0.046903 & 0.051704 & $+0.004801$ \\
1 & 1 & 1 & 1 & 0.077213 & 0.072412 & $-0.004801$ \\
\bottomrule
\end{tabular}
\end{table}

\paragraph{Why this example matters more than its magnitude.} The deviations in each $(z_1,z_2,z_3)$ block are equal and opposite, so summing $Y$ out of the library's expression recovers ground truth to $8\times10^{-17}$. Its answer is therefore correct for $P_x(z_1,z_2,z_3)$ and wrong only for the joint that was asked for. An evaluation checking a marginal, or a single scalar contrast, would see nothing. The error is confined to the factor for $Y$, which the returned expression writes as $P(y \mid z_2, z_3)$. That factor conditions on neither $z_1$, a directed parent of $Y$, nor on $x$. We record this as an observation about the expression and claim no diagnosis of the implementation.

This is what Theorem~\ref{thm:grading} buys in practice. The library's verdict is correct, its expression is a proper distribution, its marginal is exact, and it is still not the requested quantity. Only evaluation against an SCM whose interventional distribution is known separates the two.

\paragraph{Provenance.} The graph is transcribed from the printed \texttt{graph.formula} definition in the source, not from a secondary description. Parameters for the seeded model, the full-precision values, and the same table at $\Do{x{=}0}$ are in the artifact.

\paragraph{A refuted estimand from a model.} For $P(v_4 \mid \Do{v_0})$ on a five-vertex random instance, \modelPro{} returned
\[
\hat\est = \textstyle\sum_{v_2,v_3} P(v_2 \mid v_0, v_3)\, P(v_3)\, P(v_4 \mid v_0, v_2, v_3),
\]
while \ID{} returns
\[
\est = \textstyle\sum_{v_2,v_3} P(v_2 \mid v_0, v_3) \sum_{v_0'} P(v_0' \mid v_3)\, P(v_3)\, P(v_4 \mid v_0', v_2, v_3).
\]
The two differ by one summation: $\est$ averages the last factor over the observed distribution of the treatment, while $\hat\est$ evaluates it at the intervened value. The verifier refutes $\hat\est$ with a deviation of $1.2 \times 10^{-3}$. All four refuted model estimands share this error, a missing summation over the treatment.

\subsection{A published graph: Robins, Hern\'an and Brumback (2000)}
\label{app:ex-published}
The instances above are generated. To show the same machinery applied to a graph an author drew and published, we take Figure~1 of \citet{robins2000marginal}, the methods paper introducing marginal structural models, whose companion study estimates the effect of zidovudine on survival in HIV-positive men~\citep{ryalen2026causal}. We chose it because its authors mark latent variables explicitly: the prose defines $U_k$ as all unmeasured causal risk factors for the outcome, and states that panel 1b differs from 1a only in that the arrows from those risk factors into the treatment variables are removed. Latent variables appear in a substantial minority of applied graphs, roughly \csLatentShare{} of the applied health DAGs reviewed by \citet{tennant2021use}, so this is a real class of instance and not a curiosity. Causal graphs also support clinical analyses such as explaining the progression of mild cognitive impairment~\citep{behnam2024causal}, and any such graph can be checked for identifiability the same way. We transcribed the figure from the scanned original; two edges we could not settle were run as four variants, and every verdict and every frontier below is identical across them. Figure~\ref{fig:robins} shows our encoding.

\begin{figure}[h]
\centering
\begin{tikzpicture}[
  font=\scriptsize,
  v/.style={circle, draw, semithick, minimum size=0.56cm, inner sep=0pt, fill=white},
  lat/.style={v, dashed, fill=gray!15},
  d/.style={-{Stealth[length=1.6mm]}, semithick},
  rm/.style={d, red!70!black, thick},
  ttl/.style={font=\scriptsize\bfseries, text=gray!60!black},
  fbox/.style={rectangle, rounded corners=3pt, draw=green!45!black, fill=green!10,
               align=center, inner sep=3pt, text width=5.6cm}
]
\node[ttl] at (2.70,2.30) {(a) Panel 1a: unmeasured risk factors};
\node[v]   (aL0) at (0.00,0) {$L_0$};
\node[v]   (aA0) at (1.35,0) {$A_0$};
\node[v]   (aL1) at (2.70,0) {$L_1$};
\node[v]   (aA1) at (4.05,0) {$A_1$};
\node[v]   (aY)  at (5.40,0) {$Y$};
\node[lat] (aU0) at (0.90,1.55) {$U_0$};
\node[lat] (aU1) at (3.45,1.55) {$U_1$};
\draw[d]  (aU0) -- (aL0);  \draw[rm] (aU0) -- (aA0);  \draw[d] (aU0) -- (aL1);
\draw[d]  (aU0) -- (aU1);  \draw[rm] (aU0) -- (aA1);  \draw[d] (aU0) -- (aY);
\draw[d]  (aL0) -- (aA0);
\draw[d]  (aL0) to[bend right=35] (aL1);
\draw[d]  (aL0) to[bend right=40] (aA1);
\draw[d]  (aL0) to[bend right=45] (aY);
\draw[d]  (aA0) -- (aL1);  \draw[d] (aA0) to[bend right=35] (aA1);  \draw[d] (aA0) -- (aU1);
\draw[d]  (aU1) -- (aL1);  \draw[rm] (aU1) -- (aA1);  \draw[d] (aU1) -- (aY);
\draw[d]  (aL1) -- (aA1);  \draw[d] (aL1) to[bend right=35] (aY);
\draw[d]  (aA1) -- (aY);
\begin{scope}[xshift=7.4cm]
\node[ttl] at (2.70,2.30) {(b) Panel 1b: the authors' assumption};
\node[v]   (bL0) at (0.00,0) {$L_0$};
\node[v]   (bA0) at (1.35,0) {$A_0$};
\node[v]   (bL1) at (2.70,0) {$L_1$};
\node[v]   (bA1) at (4.05,0) {$A_1$};
\node[v]   (bY)  at (5.40,0) {$Y$};
\node[lat] (bU0) at (0.90,1.55) {$U_0$};
\node[lat] (bU1) at (3.45,1.55) {$U_1$};
\draw[d]  (bU0) -- (bL0);  \draw[d] (bU0) -- (bL1);
\draw[d]  (bU0) -- (bU1);  \draw[d] (bU0) -- (bY);
\draw[d]  (bL0) -- (bA0);
\draw[d]  (bL0) to[bend right=35] (bL1);
\draw[d]  (bL0) to[bend right=40] (bA1);
\draw[d]  (bL0) to[bend right=45] (bY);
\draw[d]  (bA0) -- (bL1);  \draw[d] (bA0) to[bend right=35] (bA1);  \draw[d] (bA0) -- (bU1);
\draw[d]  (bU1) -- (bL1);  \draw[d] (bU1) -- (bY);
\draw[d]  (bL1) -- (bA1);  \draw[d] (bL1) to[bend right=35] (bY);
\draw[d]  (bA1) -- (bY);
\end{scope}
\node[ttl] at (6.40,-1.85) {(c) Two minimal restoring sets for $P(y \mid \Do{a_1})$ in (a), each of size four};
\node[fbox] at (2.70,-2.60) {$\{A_0\!\leftrightarrow\!A_1,\ A_1\!\leftrightarrow\!L_0,\ A_1\!\leftrightarrow\!L_1,\ A_1\!\leftrightarrow\!Y\}$\\[1pt]\emph{the treatment $A_1$ is unconfounded}};
\node[fbox] at (10.10,-2.60) {$\{A_0\!\leftrightarrow\!Y,\ A_1\!\leftrightarrow\!Y,\ L_0\!\leftrightarrow\!Y,\ L_1\!\leftrightarrow\!Y\}$\\[1pt]\emph{the outcome $Y$ is unconfounded}};
\end{tikzpicture}
\vspace{-2mm}
\caption{Our encoding of Figure~1 of \citet{robins2000marginal}. $A_t$ is zidovudine treatment at time $t$, $L_t$ the measured risk factors, $Y$ the outcome, and $U_t$ the unmeasured risk factors (dashed, latent). \textbf{(a)} Projecting out $U_0$ and $U_1$ confounds every pair of observed variables, so the joint effect and both single effects are non-identifiable. \textbf{(b)} The authors' assumption removes the three red edges, from unmeasured risk factors into treatments; only $L_0$, $L_1$ and $Y$ stay confounded, and all three effects become identifiable. \textbf{(c)} For the effect of $A_1$ in (a), two minimal restoring sets have the same size and opposite meaning, so cardinality cannot choose between them. Two edges we could not read from the scan, $A_0 \to Y$ and $L_0 \to U_1$, change neither projection and are omitted.}
\label{fig:robins}
\end{figure}
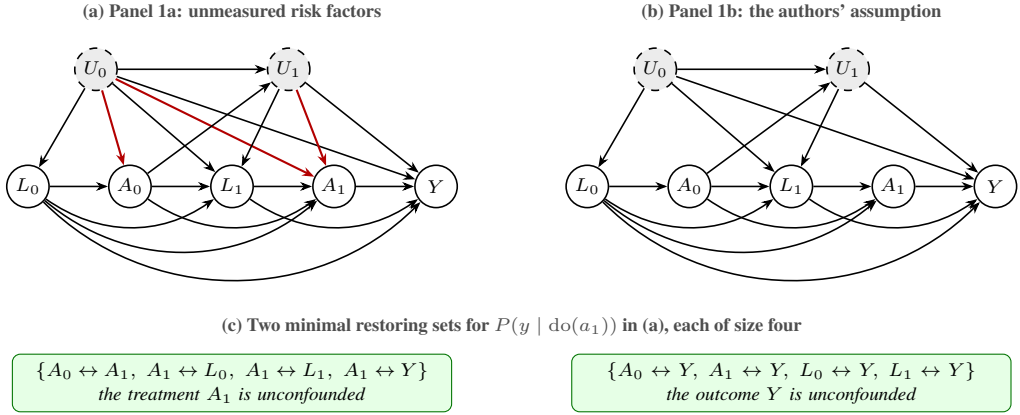

{\bf The authors' claims check out.} Both implementations agree on all six queries: the joint effect and both single effects are non-identifiable in panel 1a and identifiable in panel 1b. On 1b the authors' g-formula matches exact SCM ground truth at \csGfromulaErr{}, while the naive estimand they warn against is wrong by \csNaiveErr{} on the same models. Their central methodological point is confirmed numerically on their own graph.

{\bf Their identifying assumption is sufficient but not minimal.} The 1a-to-1b move deletes \csAuthorDeletes{} bidirected edges, every one incident to a treatment. For the joint effect the frontier has \csFrontierJoint{} elements, and one is the authors' set with $A_0 \leftrightarrow A_1$ restored: an unmeasured common cause of the two treatment decisions is harmless once both are intervened on. Every route deletes the two treatment--outcome edges, which is the necessary core. For the single effect $\Do{a_1}$ the frontier has \csFrontierSingle{} elements, two of which have the same cardinality and opposite meaning: one makes the treatment $A_1$ unconfounded, the other the outcome $Y$ (Figure~\ref{fig:robins}c). Cardinality cannot separate them, and domain judgment must.

One semantic caveat belongs with this. Frontier elements are defined over the ADMG, and a single latent in the authors' DAG can induce several bi-directed edges; deleting one and keeping another therefore posits a different latent structure and does not erase one of their arrows. This is correct for $\Adm(\Gr,\Qy)$ as defined in Sec.~\ref{sec:repair-labels} and worth stating before a frontier element is read back as advice.
\section{Discussion}
\label{app:discussion}

\subsection{What the results do and do not say about causal competence}
\label{app:disc-competence}

Deciding identifiability on a random ADMG with neutral variable names draws on no memorized association between variables and no domain knowledge attached to their names, because there are none to draw on. The accuracy we observe on that family is therefore difficult to square with accounts of language-model causal ability as retrieval~\citep{chi2024unveiling, zecevic2023causal} or as knowledge extraction from metadata~\citep{kiciman2023causal}. Those accounts were established on an earlier model generation, and the most we claim is that they warrant re-examination against current models on tasks where the retrieval route is closed by construction.

The bound in the other direction is firmer. Identification given a graph is a formal task on which a sound and complete polynomial algorithm happens to exist. That is unusual. Competence here implies nothing about causal reasoning where no such procedure is available: discovering the graph, choosing what to measure, deciding whether an assumption is credible. The small open-weights models of Appendix~\ref{app:proto-small} were tested only on edge classification, so this paper says nothing about where the ability appears as models grow. A reader who takes from this paper that language models can do causal inference has taken more than the measurement supports.

\subsection{Using these results in practice}
\label{app:disc-practice}

The practical consequences run in two directions, and the more important one is not about evaluation at all. For anyone with a graph and a query, the results argue for calling \ID{}. The algorithm is polynomial, sound and complete, and freely implemented. It decides in milliseconds a question on which the best model we tested still asserts an answer to one in a hundred unanswerable queries. Where a language model helps is in the surrounding work: reading a graph out of a paper, proposing which confounders to consider, explaining why a query fails. With \ID{} as a tool, a model's repairs are valid and minimal at every size we tested, including sizes where our exhaustive enumeration is infeasible (Sec.~\ref{sec:res-more}). The configuration to avoid is the one our main results measure: a model deciding identifiability unaided and reporting a confident estimand.

For anyone evaluating such a system, three things follow. Report the false-claim rate separately from accuracy, which does not determine it (Corollary~\ref{cor:acc-fc}). On our instances the false-claim rate spans seventeen-fold while accuracy spans ten points, so a ranking by accuracy presents the models as close competitors. Measure on your own distribution of graphs, because degradation with graph size appeared in one model family and not another, so it cannot be inferred from published numbers. And re-run before believing a small gap. At temperature zero, both endpoints we re-sampled return non-unanimous verdicts on a few percent of instances, which exceeds the accuracy gaps between our systems on the main pool (Table~\ref{tab:verdict}).

When a query is not identifiable, the frontier names the assumption sets that would restore it, and both the published graph of Appendix~\ref{app:ex-published} and the worked example of Appendix~\ref{app:ex-label} show the same thing about how that output should be read. Frontier elements of equal size can carry incompatible meanings, so cardinality does not rank them and no graphical computation can. The frontier narrows the question to a short list of substantive assumptions; choosing among them is domain work and remains so.

\subsection{Applications of this work}
\label{app:disc-applications}

{\bf Language model research.} Two questions are open in language model research. The first is whether a model can tell that a question has no answer. The second is whether a model's stated reasoning reflects how the model reached its verdict. Work on model self-knowledge and abstention studies the first~\citep{kadavath2022language, kirichenko2026abstentionbench}, and the chain-of-thought faithfulness literature studies the second~\citep{turpin2023language, lanham2023measuring, chen2025reasoning}. Both lack certified ground truth. In abstention work, whether a question is answerable is fixed by human judgment, so a model that declines may be right while its label is wrong. In faithfulness work, the correct reasoning is unknown, so a stated justification can be checked for consistency but not for truth. \benchname{} removes that ambiguity for one well-defined domain. Unanswerability here is a theorem, instances can be regenerated after any training cutoff, and the false-claim rate isolates the error an abstention study most needs to measure. The warrant results of Sec.~\ref{sec:res-warrant} give faithfulness research a second tool: because the obstruction that blocks identification is known exactly, a model's stated justification can be checked against the true one.

{\bf Cognitive and behavioral modeling of language models.} A growing line of work studies language models as subjects of psychological experiment, using paradigms from cognitive science to characterize behavior beyond aggregate accuracy~\citep{binz2023using, hagendorff2023machine, coda2024cogbench}. Three levels can be distinguished in such work: what a model does, what the model reports about its reasoning, and the internal state that produces the behavior. Existing tasks typically fix the correct response but not the correct reasoning, so self-report can be compared with behavior but not with the truth. \benchname{} fixes both. The certified label is ground truth for the verdict, and the graphical obstruction is ground truth for the justification. Sec.~\ref{sec:res-warrant} reports a dissociation between hidden reasoning and self-report: models reason internally at length yet emit a bare verdict, and name the correct obstruction only when asked. The task also bears on whether a model forms an internal model of the structure it reasons about. Random graphs with neutral names offer no memorized association, so correct verdicts there are behavioral evidence that the graph itself is represented. Probing internal activations against certified labels could test that hypothesis directly.

{\bf Human--AI interaction and reliance.} In a conversation with a language model, a user sees only what the model emits. Reliance on that output is appropriate when the person accepts correct advice and rejects incorrect advice~\citep{schemmer2023appropriate}. The older distinction between misuse and disuse of automation points toward the same target~\citep{parasuraman1997humans}. Over-reliance, the acceptance of incorrect advice, is common and resists simple remedies such as explanations~\citep{buccinca2021trust}. A false claim on an unanswerable causal query is such a failure: high accuracy conceals the error, and no observational data can reveal one. Our results sharpen the problem for longitudinal use, in which the same person returns to a system over time and the human--AI relationship develops through repeated interaction. First, a correct verdict arrives by default with no reason, so the user has nothing to check against. Second, trust is ordinarily calibrated from observed outcomes, and an error that no observational data can expose produces none. Repeated interaction therefore cannot teach the user when to distrust the system on exactly these queries. Third, verdicts vary across repeated requests at temperature zero (Appendix~\ref{app:proto-repro}), so one question may receive different answers in successive sessions. Certified labels make each of these failures observable, so appropriate reliance can be measured where reliance matters most.

{\bf Training with verifiable rewards.} Reinforcement learning with verifiable rewards trains a language model against a deterministic check of its output, and the approach underlies recent reasoning models~\citep{lambert20243, shao2024deepseekmath, guo2025deepseek}. A deterministic check is harder to exploit than a learned reward model, which invites reward hacking~\citep{skalse2022defining}. The approach is confined, however, to domains where such a check exists, which in practice means mathematics and code. Identification, including the decision that a query has no answer, has had no such check, and \benchname{} supplies one. The certified label verifies the verdict, and the numerical evaluation verifies the estimand. Instances can be generated without limit and after any training cutoff, so a held-out set stays uncontaminated. Two features of the setting bear on reward design. First, a correct estimand admits many equivalent forms, so a reward based on answer matching would penalize correct outputs; numerical grading does not. Second, the two errors are not equally costly, and a reward can weight a claimed answer to an unanswerable query more heavily than a refusal of an answerable one. One caveat applies. Theorem~\ref{thm:grading} guarantees refutation only on the family of binary SCMs the verifier samples (Remark~\ref{rem:family}). A policy optimized against the verifier could therefore learn expressions that agree on that family and differ elsewhere. Widening the family and the compared intervention values narrows that gap.

\subsection{Limitations}
\label{app:disc-scope}

The scope limits are structural and not incidental. Every instance supplies the graph, so nothing here measures whether a model can recover structure from data or from text. That step is where most applied work spends its effort, and where errors are likeliest. Identifiability is settled before any data exist, so a correct verdict says nothing about estimation, finite-sample behavior, positivity, or whether the modeling assumptions hold. The labels certify what is recoverable in principle and are silent on whether the graph is right: a certified verdict on a wrong graph is a certified answer to the wrong question. Finally, the grading compares expressions numerically at a single intervention value over two draws (Sec.~\ref{sec:verifier}, Appendix~\ref{app:properties-grading}), so it certifies the slice compared and not the whole interventional distribution. A good score on \benchname{} is therefore evidence that a system decides a formal question correctly, and nothing more.

The empirical claims have narrower limits. We evaluate three frontier models from two vendors, each at its vendor's default reasoning depth, which differs between them; the two Gemini endpoints are floating aliases and cannot be re-run against identical weights. Estimands are graded only on graphs of at most 9 vertices, which leaves \estimandSkip{} of correct claims unchecked, and \modelGPT{}'s estimands are not graded. Every figure is a single sample, and repeated sampling at temperature zero returns non-unanimous verdicts on \repeatFlipRate{} of instances (Appendix~\ref{app:proto-repro}). Theorem~\ref{thm:grading} holds relative to binary SCMs (Remark~\ref{rem:family}).

\end{document}